\documentclass{article}
\usepackage[utf8]{inputenc}
\usepackage{amsfonts}
\usepackage{mathtools}
\usepackage{graphicx}
\usepackage{hyperref}
\usepackage{enumitem}   
\usepackage{amssymb,amsmath,amsthm}
\usepackage{nicematrix}
\usepackage{multirow}
\usepackage{subcaption}
\usepackage[square]{natbib}
\usepackage{dsfont}
\usepackage[a4paper, total={6in, 8in}]{geometry}

\usepackage{cleveref}

\crefname{equation}{}{}

\title{Sampling Theorems for Unsupervised Learning in Linear Inverse Problems}
\author{Julián Tachella$^1$, Dongdong Chen$^2$ and Mike Davies$^2$}

\date{%
    $^1$ENSL, CNRS, Laboratoire de Physique, Lyon, France\\%
    $^2$University of Edinburgh, School of Engineering, UK\\[2ex]%
}

\newcommand{\signalset}{\mathcal{X}}
\newcommand{\yset}{\mathcal{Y}}
\newcommand{\baseset}{\mathcal{X}_0}
\newcommand{\ntransf}{|\mathcal{G}|}

\newcommand{\red}[1]{\textcolor{black}{#1}}
\newcommand{\JT}[1]{\textcolor{black}{#1}}
\newcommand{\R}[1]{\mathbb{R}^{#1}}

\newcommand{\C}[1]{\mathbb{C}^{#1}}
\newcommand{\vect}[1]{\text{vec}(#1)}
\newcommand{\rangeA}{\mathcal{R}_A}
\newcommand{\rangeAg}{\mathcal{R}_{A_g}}
\newcommand{\nullA}{\mathcal{N}_A}
\newcommand{\rk}[1]{\text{rank}(#1)}
\newcommand{\bdim}[1]{\operatorname{boxdim}\left(#1\right)}
\newcommand{\nullAg}{\mathcal{N}_{A_g}}
\newcommand{\Hilb}{\mathcal{H}}
\newcommand{\group}{\mathcal{G}}

\DeclareMathOperator*{\argmin}{arg\,min}

\newtheorem{theorem}{Theorem}

\newtheorem{lemma}[theorem]{Lemma}
\newtheorem{conjecture}[theorem]{Conjecture}
\newtheorem{proposition}[theorem]{Proposition}

\theoremstyle{definition}
\newtheorem{example}{Example}[section]

\begin{document}

\title{Sensing Theorems for Unsupervised Learning in Linear Inverse Problems}
\author{Julián Tachella$^1$, Dongdong Chen$^2$ and Mike Davies$^2$}

\date{%
    $^1$CNRS, ENSL, Laboratoire de Physique, Lyon, France\\%
    $^2$University of Edinburgh, School of Engineering, UK\\[2ex]%
}

\maketitle

\begin{abstract}
    Solving \JT{an ill-posed} linear inverse problem requires knowledge about the underlying signal model. In many applications, this model is a priori unknown and has to be learned from data.  However, it is impossible to learn the model using observations obtained via a single incomplete measurement operator, as there is no information \JT{about the signal model in the nullspace} of the operator, resulting in a chicken-and-egg problem: to learn the model we need reconstructed signals, but to reconstruct the signals we need to know the model. Two ways to overcome this limitation are using multiple measurement operators or assuming that the signal model is invariant to a certain group action. In this paper, we present necessary and sufficient \JT{sensing} conditions for learning the signal model from \JT{measurement data alone} which only depend on the dimension of the model and the number of operators or properties of the group action that the model is invariant to. As our results are agnostic of the learning algorithm, they shed light into the fundamental limitations of learning from incomplete data and have implications  in a wide range set of practical algorithms, such as dictionary learning, matrix completion and deep neural networks.
\end{abstract}
\vspace{.5mm}
\section{Introduction}
Inverse problems are ubiquitous in science and engineering applications such as computed tomography (CT)~\citep{jin2017deep}, depth ranging~\citep{rapp2020advances} and non-line-of-sight imaging~\citep{o2018confocal}. In this paper, we consider linear inverse problems that consist of reconstructing a signal $x\in\signalset\subset \mathbb{R}^{n}$ from incomplete and noisy measurements $y\in\JT{\yset}\subseteq\mathbb{R}^{m}$, that is
\begin{equation}
    y = Ax + \epsilon
\end{equation}
where $\epsilon$ denotes the noise affecting the measurements. This is generally an ill-posed task due to the incomplete forward operator $A$ with $m< n$ and the presence of noise.
Classical approaches assume a signal \JT{model} using some prior knowledge about the underlying signals. For example, the well-known total variation model~\citep{rudin1992nonlinear} is built on the prior belief that natural images are approximately piecewise smooth.  This approach often yields a loose description of the true model, providing biased and/or suboptimal reconstructions. In recent years, this approach has been replaced by learning \JT{the signal model} directly from data. For example, a common approach is to learn a model consisting of a dictionary of image patches using a dataset of natural images~\citep{studer2012dictionary}. In a similar vein, it is possible to learn directly the reconstruction function $f:y\mapsto x$ via deep neural networks using multiple training pairs $(x_i,y_i)$~\citep{jin2017deep}.

Despite the appeal and better performance of the learning-based approach, in many sensing applications such as medical and scientific imaging we can only access measurements~\JT{$y_i$ which hinders supervised learning solutions.} Moreover, if the measurement process $A$ is incomplete, it is fundamentally impossible to identify the \JT{signal set\footnote{\JT{Throughout the paper we use the terms \emph{signal model} and \emph{signal set} interchangeably to denote the support $\signalset$ of the signal distribution $p(x)$.}} $\signalset$} through measurements alone \JT{even in the absence of noise}, as there is no information about the set of signals in the nullspace of $A$. \JT{In other words, we cannot uniquely identify the signal set $\signalset$ from the noiseless measurement set $\yset=A\signalset$. It is also impossible to learn the reconstruction function $f:y\mapsto x$, independently of the number of observed measurement vectors $y_i$, as explained by the following simple proposition:}

\begin{proposition} \label{prop:simple} \citep{chen2021equivariant}
\JT{Any reconstruction function $f:\mathbb{R}^{m}\mapsto\mathbb{R}^{n}$ of the form
\begin{equation}
\label{eq:unlearnable}
    f(y) = A^{\dagger}y + v(y)
\end{equation}
where $A^{\dagger}$ is the linear pseudo-inverse of $A$ and $v(y)$ is any function whose image belongs to the nullspace of $A$ verifies measurement consistency $Af(y_i)=y_i$. }
\end{proposition} 

\begin{proof} \JT{For any $f$ of the form \cref{eq:unlearnable}, the measurement consistency can be expressed as $Af(y) = AA^{\dagger}y + Av(y)$ where the first term is simply $y$ as $AA^{\dagger}$ is the identity matrix, and $Av(y)=0$ for any $v(y)$ in the nullspace of $A$.}
\end{proof}
Hence the chicken-and-egg problem: in order to reconstruct $x$ we need the signal model, but to learn this model we require some ground truth training data $x_i$.

Here we show that this fundamental limitation can be overcome \JT{if $\signalset$ is low-dimensional,} either by  using information from multiple incomplete \JT{measurement} operators $A_1,\dots,A_{\ntransf}$ or by having a single operator $A$ and exploiting weak prior information associated with group invariance properties of the signal model. 

Multiple measurement operators can provide additional information about the  model if the operators have different nullspaces. This idea was used in various inverse problems such as image inpainting~\citep{studer2012dictionary}, magnetic resonance imaging (MRI)~\citep{liu2020rare} and hyperspectral imaging~\citep{yang2015mixture}.
Recently, the equivariant imaging framework~\citep{chen2021equivariant,chen2021robust} empirically showed that learning the signal model with a single operator $A$ is also possible if the signal model is invariant to a certain group of transformations\footnote{Formally, these transformations constitute the action of a group (also called the representation of a group).} $T_1,\dots,T_{\ntransf}$, as they give access to multiple \emph{virtual} operators $AT_1,\dots,AT_{\ntransf}$ with possibly different nullspaces. 
This strategy offers an appealing way to learn the model, as most real-world signal sets present certain invariances, such as invariance to translations and rotations in images.
\begin{table}[t]
\begin{tabular}{|l|l|l|l|}
\hline
\multicolumn{1}{|c|}{\textbf{Inverse problem}} & \multicolumn{1}{c|}{\begin{tabular}[c]{@{}l@{}} \textbf{Forward} \\  \textbf{operator}\end{tabular}} & \multicolumn{1}{c|}{ \begin{tabular}[c]{@{}l@{}} \textbf{Group action}/ \\  \textbf{Mult. operator}\end{tabular} } & \multicolumn{1}{c|}{\textbf{Algorithm}} \\ \hline
Image inpainting & Binary mask & \begin{tabular}[c]{@{}l@{}}Permutations\\ 2D shifts\end{tabular} & \begin{tabular}[c]{@{}l@{}}Dictionary learning \\\citep{studer2012dictionary} \end{tabular}  \\ \hline
\begin{tabular}[c]{@{}l@{}}Computed \\  tomography\end{tabular} & \begin{tabular}[c]{@{}l@{}}Sparse-view\\  Radon transform\end{tabular} & 2D rotations &
\begin{tabular}[c]{@{}l@{}}Equivariant Imaging \\\citep{chen2021equivariant,chen2021robust}\end{tabular} \\ \hline
\multirow{2}{*}{MRI} & \multirow{2}{*}{
\begin{tabular}[c]{@{}l@{}}Subsampled \\Fourier\end{tabular}}& Multiple op. & 
\begin{tabular}[c]{@{}l@{}}Deep networks \\\citep{liu2020rare}\end{tabular}
 \\ 
\cline{3-4}
 & & 2D rotations & \begin{tabular}[c]{@{}l@{}}Equivariant Imaging \\\citep{chen2021robust}\end{tabular}  \\
\hline
High-speed video & Binary mask & 2D shifts & 
\begin{tabular}[c]{@{}l@{}}Low-rank GMM \\\citep{yang2015mixture}\end{tabular}
\\ \hline
Depth completion & Binary mask & Multiple op. & 
\begin{tabular}[c]{@{}l@{}}Deep networks \\\citep{jaritz2018depth}\end{tabular}
\\ \hline 
\begin{tabular}[c]{@{}l@{}}Hyperspectral \\imaging\end{tabular}
  & \begin{tabular}[c]{@{}l@{}}Spectral domain\\  CS matrix\end{tabular} & Multiple op. & \begin{tabular}[c]{@{}l@{}}Low-rank GMM \\\citep{yang2015mixture}\end{tabular} \\ \hline
  \begin{tabular}[c]{@{}l@{}}Electron \\microscopy\end{tabular}
  & 2D projection & \begin{tabular}[c]{@{}l@{}}3D rotations \\and  shifts\end{tabular} & \begin{tabular}[c]{@{}l@{}}Cryo-GAN \\\citep{gupta2020multi}\end{tabular} \\ \hline
\begin{tabular}[c]{@{}l@{}}Motion \\segmentation\end{tabular} & Binary mask & Permutations & SSC~\citep{yang2015sparse} \\ \hline
\begin{tabular}[c]{@{}l@{}}Single image \\ view synthesis\end{tabular} & Light transport & 3D rotations & \begin{tabular}[c]{@{}l@{}}Coordinate-based \\ neural representation \\\citep{tancik2021learned}
\end{tabular} \\ \hline 
\JT{\begin{tabular}[c]{@{}l@{}}Compressed \\ Sensing
\end{tabular}} & \JT{Random Gaussian} & \JT{Permutations} & \begin{tabular}[c]{@{}l@{}} \JT{Approximate Message} \\ \JT{Passing} \\ \JT{\citep{guo2015near}}\\\JT{\citep{metzler2018unsupervised}}
\end{tabular} \\\hline 
\end{tabular}
\caption{Examples of applications where a low-dimensional signal model is learned via multiple measurement operators or assuming that the model is invariant to a group action.}
\label{tab:applications}
\end{table}

\Cref{tab:applications} presents various inverse problems and related reconstruction algorithms where information in the nullspace of $A$ is obtained using multiple measurement operators or exploiting invariance of the signal model to the action of a group. Despite the empirical successes of these methods, theoretical guarantees for model identification are still lacking: what are the requirements on the measurement operators or model invariance properties? When is it possible learn the model and reconstruct the signals? The two fundamental problems of signal recovery and model identification can be \JT{formally defined} as follows: 
\begin{description}
    \item[Signal Recovery] \JT{Signals $x$ can be uniquely recovered from observations $y=Ax$ if for every measurement vector $y\in\yset$ there is a unique signal $x\in\signalset$ that verifies the observed measurements $y=Ax$.}
    \item[Model Identification] \JT{We define a signal set $\signalset$  to be identifiable within a given a class~$\mathcal{C}$ if it is uniquely defined from the observed measurement sets $\yset_g = A_g \signalset$ with $g=1,\dots,\ntransf.$}
\end{description}
In general, there can be a unique solution for none of the problems, just one or both. There might be a unique solution for signal recovery if the signal model is known, but it might be impossible to \JT{identify} the model 
in the first place. For example, this is the case of blind compressed sensing~\citep{gleichman2011blind}. The converse is also possible, that is, uniquely identifying a model without having enough measurements per sample to uniquely identify the associated signal. For example, multiple single ($m=1$) measurements  $y_i=a_{g_i}^{\top}x_i$ are sufficient to identify a low-dimensional signal subspace~\citep{chen2015rankone,gribonval2017compressive}, however it is not possible to recover the signal $x_i$ linked to each observed measurement $y_i$. 

It is well-known from generalized compressed sensing theory that unique signal recovery is possible if the nullspace of $A$ does not contain any vector in the \red{difference set} of $\signalset$~\citep{bourrier2014fundamental}. This condition can be achieved by \JT{the class of \emph{low-dimensional signal sets}}, i.e., where the  dimension of $\signalset$ is  both smaller than the ambient dimension $n$ and observed measurements $m$. \JT{For example, in standard compressed sensing signals are assumed to belong the $k$-sparse set, where $\signalset$ has dimension $k<m$.}  
On the other hand, to the best of our knowledge, the model identification problem has been mostly analysed in the context of matrix completion~\citep{candes2009exact}, where $\signalset$ is assumed to belong the class of low-dimensional subspaces or union of subspaces~\citep{eriksson2012high}, and \JT{the observations $y_i=A_{g_i}x$ consist of} different permutations of a binary mask, i.e., $A_g=AT_g$ where $A$ is a diagonal matrix indicating the observed entries and $T_1,\dots,T_{\ntransf}$ are permutation matrices. 

In this paper, we study necessary and sufficient conditions \JT{on the number of measurements} to guarantee a unique solution to both problems for \JT{the general class of low-dimensional} signal models, only relying on their intrinsic dimension and weak prior information related to invariance properties of the signal model.


\subsection{Summary of Contributions}

We consider the class of $k$-dimensional signal sets (the definition of dimension is made clear in \Cref{subsec: low-dim models}), and study model identifiability from incomplete measurements  associated with a finite number of measurement operators \JT{where each training sample  consists of an observation of a signal from the set $\signalset$ through one of a set of $\ntransf$ measurement operators}, i.e., $y_i = A_{g_i} x_i+\epsilon_i$ with $g_i\in \{1,\dots,\ntransf\}$, where $A_{g}\in \R{m\times n}$ and $m<n$. We also consider the case where we have a single operator $A$, i.e., $y_i = A x_i+\epsilon_i$, but the signal set is invariant to the action of a group $\group$. Our main results are:

\begin{itemize}
\item When the dimension of the model is large, i.e., $k=n$, the signal model cannot be uniquely identified if $m<n$.
\item When the dimension of the model is small, $k<n$, model uniqueness is possible.

\begin{itemize}
\item We first consider the case where we obtain measurements via a set of $\ntransf$ different operators $A_1,\dots,A_{\ntransf}$.
\begin{itemize}
    \item  A necessary condition for model identification is $m \geq n/ \ntransf$.
    \item Unique model identification is possible via almost every set  of $\ntransf$ operators  $A_1,\dots,A_{\ntransf}\in\R{m\times n}$ \JT{(w.r.t.\ the Lebesgue measure on $\R{\ntransf m n}$)} when $m > k + n/\ntransf$.
\end{itemize}
 \item Second, we consider the case where we have a single operator $A$ but the signal set is invariant to a group of transformations. Our results are stated as a function of $\max_j  c_j/s_j$, where $c_j$ and $s_j$ are the dimension and multiplicity of the $j$th irreducible representation of the group action\footnote{A detailed explanation of these concepts from linear representation theory can be found in \Cref{sec: group learning}.}. 
 \begin{itemize}
     \item A necessary condition for model identification for a compact group is $m \geq \max_j  c_j/s_j$.
     \item We show that, if $\group$ is a cyclic group, unique model identification is possible  by almost every operator $A\in\R{m\times n}$ \JT{(w.r.t. the Lebesgue measure on $\R{mn}$)} with $m > 2k + \max_j  c_j + 1$. \red{As cyclic groups have all irreducibles with dimension $s_j=1$, the bound is equivalent to $m > 2k + \max_j  c_j/s_j + 1$. We conjecture that this last bound  holds for more general groups with $s_j\geq 1$.}
     
    \item We characterize a subset of operators that fail to provide model uniqueness: if the forward operator $A$ is itself equivariant to the action of the group, it is impossible to uniquely identify the model if $m<n$, even for arbitrarily small $k$ and large number of transformations $\ntransf$.
 \end{itemize}
\end{itemize}
 \item We experimentally show that our bounds accurately characterize the performance of popular learning algorithms on synthetic and real datasets.
\end{itemize}
 \begin{figure}[t]
\centering
\includegraphics[width=.9\textwidth]{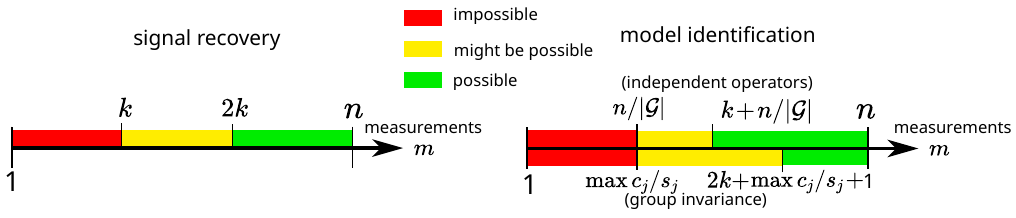}
\caption{Model identification and signal recovery regimes as a function of the number of partial observations $m$ per signal, model dimension $k$, ambient dimension $n$ and number of measurement operators $\ntransf$ (when it is possible to access multiple independent operators) or maximum multiplicity of an invariant subspace $\max_j  c_j/s_j$ (when the signal is group invariant or the operators are related via a group action).}
\label{fig:results_dc}
\end{figure}
\red{As summarized in~\Cref{fig:results_dc}, our results characterize the \JT{sensing} regimes for model identification, complementing the existing results for signal recovery. 
In particular, we provide necessary and sufficient conditions for identifying a low-dimensional \JT{group} invariant signal model from incomplete measurements of a single operator as a function of the properties of the group. 
 For example, consider 1D signals defined on a discrete grid of length $n$. We can uniquely identify a $k$ dimensional model with invariance to cyclic shifts using $m>2k+2$ measurements, whereas for a model with invariance to reflection we can only guarantee indentifiability using $m>2k+n/2+1$ measurements.
 \JT{It is important to note that our identifiability results are purely geometric and do not cover the number of samples needed to learn the model structure (the so-called sample complexity). Quantifying the sample complexity would require additional assumptions on the class of signal models, the noise and the choice of learning algorithm, and is out of scope of the present paper. However, in the experiments, we show that our bounds capture the behaviour of various learning algorithms in practical datasets.}
 }


\subsection{Related Work}
Some preliminary results of this work appear in an earlier manuscript~\citep{tachella2022sampling}, which only covers the case of learning from multiple independent operators. This section discusses other related work by topic, from blind compressed sensing to deep neural networks for inverse problems.
\paragraph{Blind Compressed Sensing}
The fundamental limitation of failing to learn a signal model from incomplete data goes back to  blind compressed sensing~\citep{gleichman2011blind} for the specific case of models exploiting sparsity on an orthogonal dictionary. In order to learn the dictionary from incomplete observations, \citep{gleichman2011blind} imposes additional constraints on the dictionary, while some subsequent papers~\citep{silva2011blind,aghagolzadeh2015new} remove these assumptions by proposing to use multiple operators $A_g$ as studied here. This paper can be seen as a generalization of such results to a much wider class of signal models.

\paragraph{Matrix Completion}
Matrix completion consists of inferring missing entries of a data matrix $Y = [y_1,\dots,y_N]$, whose columns \red{can be seen as partial observations of signals $x_i$, i.e., $y_i = A_{g_i}x_i$ where the operators $A_{g_i}$ select a random subset of $m$ entries of the signal $x_i$. In order to recover the missing entries, it is generally assumed that the signals $x_i$ (the columns of $X=[x_1,\dots,x_N]$) belong to a $k$-dimensional subspace with $k\ll n$. This problem can be viewed as the combination of  model identification, i.e., identifying the $k$-dimensional subspace, and signal recovery, i.e., reconstructing the individual columns.} If the columns are observed via $\ntransf$ sufficiently different patterns $A_{g_i}$ with the same number of entries $m$, a sufficient condition~\citep{pimentel2016characterization} for uniquely recovering almost every subspace model is\footnote{A larger number of measurements $m =\mathcal{O}(k\log n)$ is required to guarantee a stable recovery when the number of patterns $\ntransf$ is large~\citep{candes2009exact}.} $m \geq  (1-1/\ntransf)k + n/\ntransf$. 
A similar necessary condition was  shown in~\citep{pimentel2016information} for the case of \emph{high-rank} matrix completion~\citep{eriksson2012high}, which arises when the samples $x_i$ belong to a union of $k$-dimensional subspaces. 

We show that unique recovery is possible for almost every set of $\ntransf$ operators with $m > k + n/\ntransf$ measurements, however the  theory presented here goes beyond union of subspaces, being also valid for general low-dimensional models.

\paragraph{Multireference Alignment and Electron Microscopy}
In the multireference alignment problem~\citep{aguerrebere2016fundamental}, the goal is to reconstruct a signal $x$ from shifted and incomplete measurements, i.e.,
\begin{align}
    g_i &\sim \group \\
    y_i &= A T_{g_i} x + \epsilon_i
\end{align}
for $i=1,\dots,N$, where the shifts $T_{g_i}$ are a priori unknown. This problem can be generalized beyond shifts by considering transformations related to the action of a group $\group$. This is the case of electron-microcopy imaging, where  $\group$ is the group of rotations and translations of a particle~\citep{gupta2020multi}. 
The multireference alignment problem can be seen as a special case of our framework with a single operator $A$ and a $\group$-invariant signal set of dimension\footnote{For infinite but compact groups, $k=\text{dim}(\group)$.} $k=0$ which consists of all transformations of a single signal. If $\group$ is cyclic (e.g., shifts), our results guarantee model identification for almost every $A$ with $m>\max_j  c_j+1$ measurements.

\paragraph{Deep Nets for Inverse Problems}
Despite providing very competitive results, most deep learning based solvers require measurements and signal pairs $(x_i,y_i)$ in order to learn the reconstruction function $y\mapsto x$. A first step to overcome this limitation is due to Noise2Noise~\citep{lehtinen2018noise2noise}, where the authors show that it is possible to learn from only noisy data. However, their ideas only apply to denoising settings where there is a trivial nullspace, as the operator $A$ is the identity matrix.
In Artifact2Artifact~\citep{liu2020rare}, it was empirically shown that it is possible to exploit different measurement operators to learn the reconstruction function in the context of MRI. 
AmbientGAN~\citep{bora2018ambientgan} proposed to learn the signal distribution  from  incomplete measurements of multiple forward operators, however they only provide reconstruction guarantees for the case where an infinite number of operators is available\footnote{Their result relies on the Cram\'er-Wold theorem, which is discussed in~\Cref{subsec: highdim}}, $\ntransf=\infty$, a condition that is not met in practice.  
Finally, the equivariant imaging framework~\citep{chen2021equivariant} showed empirically that learning the reconstruction function using a single operator $A$ is possible in various imaging inverse problems when the signal model presents some group invariance. Here we provide a theoretical framework to support all of these findings.



\section{Signal Recovery Preliminaries} \label{subsec:CS preliminary}

Let $A^{\dagger}\in \R{n\times m}$ be the linear pseudo-inverse of $A$. 
We denote the range space of $A^{\dagger}$ as $\rangeA$. Its complement, the nullspace of $A$, is denoted as $\nullA$, where $\rangeA \oplus \nullA = \R{n}$ and $\oplus$ denotes the direct sum. Throughout the paper, we assume that the signals are associated with a distribution $p(x)$ supported on the signal set $\signalset\subset \R{n}$. Signal recovery has a unique solution if and only if the forward operator \JT{restricted\footnote{\JT{Note that if $m<n$, the operator $A$ cannot be one-to-one in the whole $\R{n}$.}} to the signal set $\signalset$}  is one-to-one, i.e., if for every pair of signals $x_1,x_2\in\signalset$ where $x_1\neq x_2$ we have that 
\begin{align}
    Ax_1 \neq Ax_2 \\
    A(x_1-x_2) \neq 0 
\end{align}
In other words, there is no vector $x_1-x_2\neq 0$ \JT{with $x_1,x_2\in\signalset$} in the nullspace of $A$.
It is well-known that this is only possible if the signal set $\signalset$ is low-dimensional. There are multiple ways to define the notion of dimensionality of a set in $\R{n}$. In this paper, we focus on the upper box-counting dimension~\citep[Chapter~2]{falconer2004fractal} which is defined for a compact subset $S\subset\R{n}$ as
\begin{equation}
   \bdim{S} = \lim \sup_{\epsilon\to0}  \frac{\log \mathds{N}(S,\epsilon)}{-\log \epsilon}
\end{equation}
where $\mathds{N}(S,\epsilon)$ is the minimum number of closed balls of radius $\epsilon$ with respect to  the norm $\|\cdot\|$ that are required to cover $S$. \JT{This definition of dimension is convenient for the theoretical results, and covers both well behaved models such as compact manifolds where the definition coincides with the more intuitive topological dimension, as well as more general sets}\footnote{\red{As discussed in~\Cref{subsec: low-dim models}, we can extend the box-counting dimension for non-compact conic sets.}}.  The \JT{sensing} mapping is one-to-one for almost every operator $A\in\R{m\times n}$ \JT{when restricted to $\signalset$} if~\citep{sauer1991embedology}
\begin{equation}\label{eq: one-to-one}
    m>\bdim{\Delta\signalset}
\end{equation}
where $\Delta\signalset$ denotes the  \red{difference set} defined as 
\begin{equation}
   \Delta\signalset = \{ \Delta x \in \R{n} | \; \Delta x =x_2 -x_1, \; 
   \; x_1,x_2\in\signalset, \; x_2\neq x_1 \}.
\end{equation}
The term \emph{almost every} means that the complement has Lebesgue measure 0 in the space of linear measurement operators $\R{m\times n}$. \JT{Almost every statements are used to exclude \emph{small} subsets of operators where one-to-oneness fails.} The difference set of models of dimension $k$  generally has dimension $2k$, requiring  $m>2k$ measurements to ensure signal recovery.
For example, the  union of $k$-dimensional subspaces requires at least $2k$ measurements to guarantee one-to-oneness~\citep{blumensath2009uos}. This includes well-known models such as $k$-sparse models  (e.g., convolutional sparse coding~\citep{bristow2013fast}) and co-sparse models (e.g., total variation~\citep{rudin1992nonlinear}).
In the regime $k<m\leq 2k$, the subset of signals where one-to-oneness fails is typically at most $(2k-m)$-dimensional~\citep{sauer1991embedology}. 
While the bound in~\Cref{eq: one-to-one} guarantees \emph{unique} signal recovery, more measurements  
are typically necessary in order to have a \emph{stable} inverse $f:y\mapsto x$, i.e., possessing a certain Lipschitz constant. A detailed discussion can be found for example in~\citep{puy2017recipes}.

\section{\JT{Model Identifiability with} Multiple Operators} \label{sec:multiple ops}

We first focus on the noiseless case to study the intrinsic identifiability problems associated to having only incomplete measurement data. The effect of noise will be discussed in~\Cref{sec: noise}. \JT{In this section, we assume that we observe the measurement sets $\yset_g=A_g\signalset$ defined as}
\begin{equation}
    \JT{\yset_g = \{ y\in\R{m} |  \;y=A_gx, \; x\in\signalset \}}
\end{equation}
\JT{for $g=1,\dots,\ntransf$. In practice, we will only observe a finite number of measurements $y_i$,} where the $i$th signal is associated to one of $\ntransf$ linear operators  with $m<n$, i.e.,
\begin{equation}
    y_i = A_{g_i} x_i
\end{equation}
where $g_i \in \{1,\dots,\ntransf\}$ and $i=1,\dots,N$. While we assume that the measurement operator $A_{g_i}$ is known for all observed signals, it is important to note that we do not know a priori if two observations $(y_i,A_{g_i})$ and $(y_{i'},A_{g_{i'}})$ are related to the same\footnote{\JT{The Noise2Noise approach~\citep{lehtinen2018noise2noise} does require such pairs of observations, which is impractical in many real-world settings. The theory presented here does not require such an assumption. }} \JT{underlying signal $x$}.  We assume that all operators have the same number of measurements $m$ throughout most of this section, however we also discuss the case of different number of measurements in~\Cref{subsec: different m}. 
We begin with a simple toy example illustrating this observation model:

\begin{example}[One-dimensional toy model]\label{ex: 2D toy}
Consider a one-dimensional linear subspace spanned by a triangular signal $\phi \in \R{n}$, i.e., $x_i= c_{i}\phi\in\signalset\subset\mathbb{R}^{n}$ with $c_{i}\in \R{}$, which is only partially observed through a simple masking operator, as illustrated in \Cref{fig:ex 3.1}. The single sensing operator does not provide information outside the observed part, thus there are infinitely many one-dimensional models that fit the observed measurements. 
However, if we also obtain measurements through multiple masks $A_{g_i}$, 
then identifying the model becomes feasible. In this case, solving both model identification and signal recovery problem boils down to low rank matrix completion of the measurement matrix $[y_1,\dots,y_N]$.
\end{example}

\begin{figure}[h]
     \centering
     \includegraphics[width=.5\textwidth]{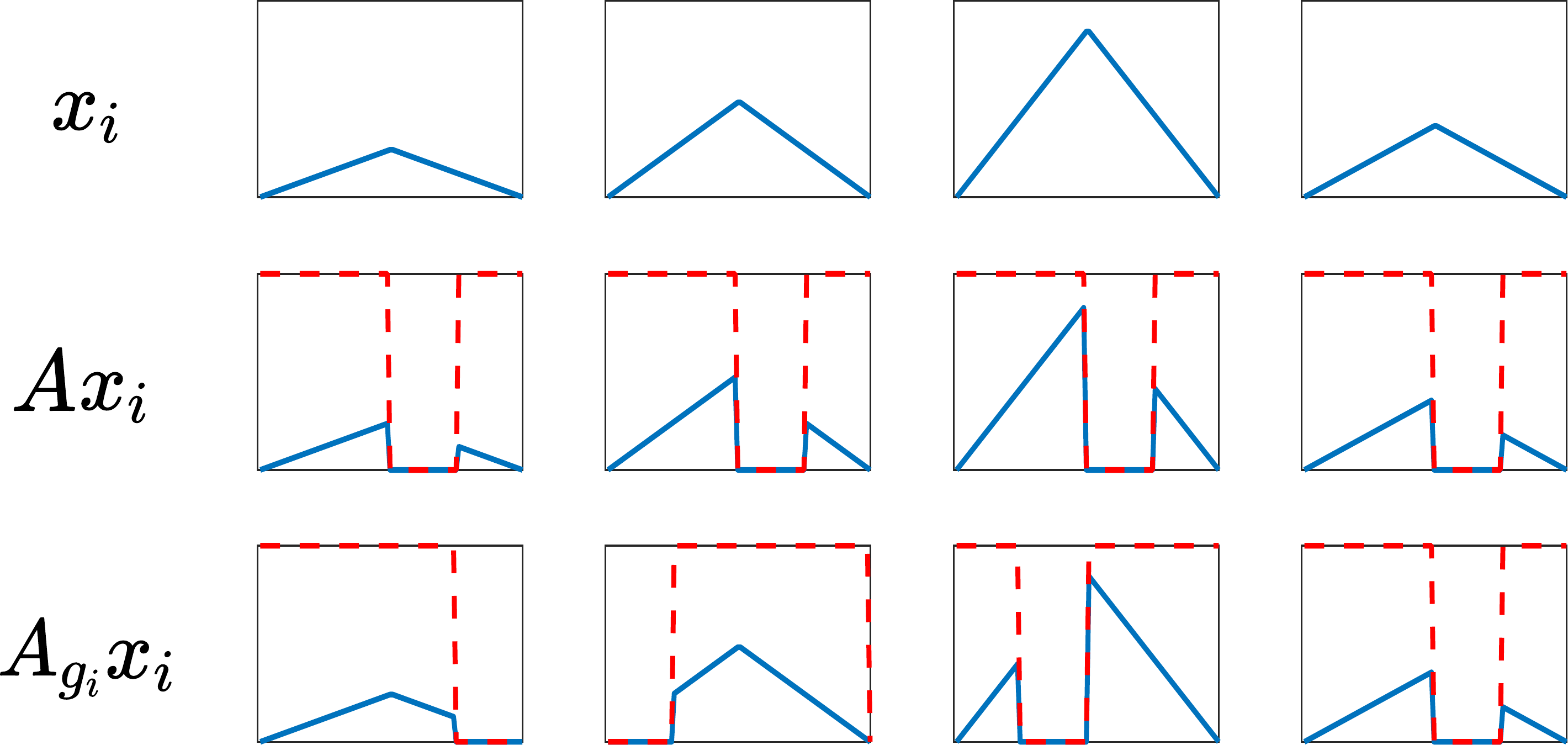}
\caption{Learning a one-dimensional model from incomplete measurement data. Each column shows a different sample from the signal set. The \JT{sensing} mask is shown in red.}
     \label{fig:ex 3.1}
\end{figure}

\subsection{Uniqueness of Any Model?} \label{subsec: highdim}
A natural first question when considering uniqueness of the signal model is: can we recover any \JT{distribution} $p(x)$ observed via a finite number of forward operators $A_g$, even in the case where the support $\signalset$ is the full ambient space $\mathbb{R}^{n}$? We show that, in general, the answer is no.

Uniqueness can be analyzed from the point of view of the characteristic function of $p(x)$, defined as $\varphi(w) = \mathbb{E}\{e^{\mathrm{i}w^{\top}x}\}$ where the expectation is taken with respect to $p(x)$ and $\mathrm{i}=\sqrt{-1}$ is the imaginary unit. If two distributions have the same characteristic function, then they are necessarily the same almost everywhere. Each forward operator provides information about a subspace of the characteristic function  as
\begin{align}
\label{eq: stat}
\mathbb{E}\{e^{\mathrm{i}w^{\top}A_g^{\dagger}y}\} 
&= \mathbb{E}\{e^{\mathrm{i}w^{\top}A_g^{\dagger}A_gx}\} \\
&= \mathbb{E}\{e^{\mathrm{i}(A_g^{\dagger}A_g w)^{\top}x}\} \\
&\JT{= \varphi(A_g^{\dagger}A_gw)}
\end{align}
\JT{where $A_g^{\dagger}A_g$ is a linear projection onto the subspace $\rangeAg$.}
Given that $m<n$, the characteristic function is only observed in the subspaces $\rangeAg$ for all $g\in \{1,\dots,\ntransf\}$. For any finite number of operators, the union of these subspaces does not cover the whole $\mathbb{R}^{n}$, and hence there is loss of information, as the signal model is not uniquely defined. 

In the case of an infinite number of operators, $\ntransf=\infty$, there is a well-known case where model uniqueness is possible:
the Cram\'er-Wold theorem guarantees uniqueness of the signal distribution if all possible one dimensional projections ($m=1$) are available~\citep{cramer1936some,bora2018ambientgan}. 
However, in most practical settings we can only access a finite set of  operators and many distributions are non-identifiable.

\subsection{Uniqueness of Low-Dimensional Models} \label{subsec: low-dim models}
Most models appearing in signal processing and machine learning are assumed to be approximately low-dimensional, with a dimension $k$ which is much lower than the ambient dimension $n$.
As discussed in~\Cref{subsec:CS preliminary}, the low-dimensional property is fundamental to obtain stable reconstructions in ill-posed problems with $m<n$. 
In the rest of paper, we impose the following assumption on the \JT{class of signal models}:
\begin{enumerate}
    \item[\textbf{A1}] The signal set $\signalset$ is either 
    \begin{enumerate}
        \item  A bounded set with box-counting dimension $k$.
        \item An unbounded conic set whose  intersection with the unit sphere has box-counting dimension $k-1$.
    \end{enumerate}
\end{enumerate}
This assumption has been widely adopted in the inverse problems literature, as it is necessary  to guarantee signal recovery. Our definition of dimension covers most models used in practice, such as \JT{single subspace models (principal component analysis),} union of subspaces (convolutional sparse coding models, $k$-sparse models), low-rank matrices and compact manifolds. \JT{It also covers signal sets that are well described by a deep generative model (e.g., a variational autoencoder or a generative adversarial network), whose  box-counting dimension is given by the dimension of the latent space. \Cref{fig: boxdim examples} shows two examples of common signal sets with small box-counting dimension. In real-world problems, we might not know the box-counting dimension of the dataset a priori. Nonetheless, there exist a range of algorithms that compute the box-counting dimension of real datasets~\citep{hein2005intrinsic} which can be applied on measurement data alone\footnote{\JT{If $A$ is one-to-one when restricted to $\signalset$ (which happens for almost every $A$ with $m>2k$ measurements, as explained in~\Cref{subsec:CS preliminary}), the box-counting dimension of the measurement data equals the box-counting dimension of the underlying signal set.}}.}

\begin{figure}[t]
     \centering
 \begin{subfigure}{0.3\textwidth}
     \centering
     \includegraphics[width=.65\textwidth]{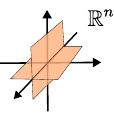}
     \caption{\JT{Union of subspaces}}
 \end{subfigure}
 \hspace{7mm}
 \begin{subfigure}{0.55\textwidth}
     \centering
     \includegraphics[width=\textwidth]{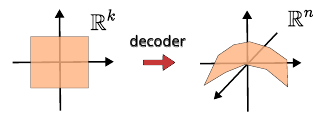}
     \caption{\JT{Generative model}}
 \end{subfigure}
\caption{\JT{Examples of low-dimensional signal sets. a) the intersection of a union of $k$-dimensional subspaces with the unit sphere has box-counting dimension $k-1$. b) the box-counting dimension of a generative model with a Lipschitz continuous decoder equals the dimension of the latent space. }}\label{fig: boxdim examples}
\end{figure}

\JT{In the rest of the paper, we focus on conditions for the identification of the support $\signalset$ instead of the signal distribution $p(x)$, due to the following observation:}
if there is a one-to-one reconstruction function \JT{$f:y\mapsto x$ (which holds for almost every operator with $m>2k$ measurements, as explained in~\Cref{subsec:CS preliminary})}, uniqueness of the support necessarily implies uniqueness of the distribution $p(x)$:
\begin{proposition}
\JT{Assume there is a one-to-one mapping between the measurement set $\yset$ and the support $\signalset$ of the signal distribution $p(x)$. Further assume that the support $\signalset$ is known, then the signal distribution $p(x)$ can be uniquely identified from the measurement distribution $p(y)$.}
\end{proposition}
\begin{proof}
If $\signalset$ is identifiable from the measurement set $\yset$ and there is a measurable one-to-one mapping from each observed signal $y_i$ to $\signalset$, then it is possible to obtain $p(x)$ as the push-forward of the measurement distribution $p(y)$.
\end{proof}

We begin with a simple example which provides important geometrical  intuition of how a low-dimensional model can be learned via multiple projections $A_g$:

\begin{example} \label{ex: 1d line}
Consider a toy signal model with support $\signalset\subset\mathbb{R}^3$ which consists of a one-dimensional linear subspace spanned by $\phi = [1,1,1]^{\top}$, and $\ntransf=3$ measurement operators $A_1,A_2,A_3\in \mathbb{R}^{2\times 3}$ which project the signals into the $x(3)=0$, $x(2)=0$ and $x(1) = 0$ planes respectively, where $x(i)$ denotes the $i$th entry of the vector $x$. 
The example is illustrated in \Cref{fig:toy_line}. The first operator $A_1$ imposes a constraint on $\signalset$, that is, every $x\in\signalset$ should verify $x(1)-x(2)=0$. Without more operators providing additional information about $\signalset$, this constraint yields a plane containing $\signalset$, and there are infinitely many one-dimensional models that would fit the measurement data perfectly. However, the additional operator $A_2$ adds the constraint $x(2)-x(3)=0$, which is sufficient to uniquely identify $\signalset$ as 
\begin{align*}
     \signalset = \hat{\signalset} := \{v \in \mathbb{R}^3 | \;  v(1) - v(2) = v(2) - v(3) = 0 \}
\end{align*}
is the desired 1-dimensional subspace.
Finally, note that in this case the operator $A_3$ does not restrict the signal set further, as the constraint $x(1)-x(3)=0$ is verified by the other two constraints.
\end{example}

\begin{figure}[t]
\centering
\includegraphics[width=1\textwidth]{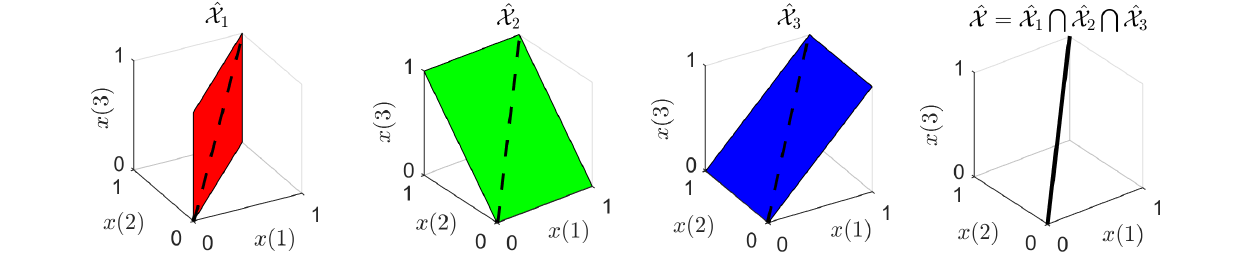}
\caption{Toy example of a 1-dimensional subspace embedded in $\R{3}$. If we only observe the projection of the signal set into the plane $x(3)=0$, then there are infinite possible lines that are consistent with the measurements. Adding the projection into the $x(1)=0$ plane, allows us to uniquely identify the signal model.}
\label{fig:toy_line}
\end{figure}

\subsection{Bounds for Independent Operators}
The ideas from~\Cref{ex: 1d line} can be generalized and formalized as follows: for each projection $A_g$, \JT{we can identify the signal set that is consistent with the observations through the measurement operator $A_g$, that is}
 \begin{equation} \label{eq: constraint Ag}
    \hat{\signalset}_g = \{v \in \mathbb{R}^{n} | \; v = \hat{x}_g + u, \; \hat{x}_g\in \signalset, \; u\in \nullAg \}
\end{equation}
which has dimension at most $n-(m-k)$. \JT{It can clearly be seen that $\signalset \subseteq \hat{\signalset}_g$ by considering the subset with $u = 0$.} The inferred signal set belongs to the intersection of these sets 
\begin{equation}
    \hat{\signalset}=\bigcap_{g\in\group} \hat{\signalset}_g 
\end{equation}
which can be expressed concisely as
\begin{equation} \label{eq:inferred set}
  \hat{\signalset} =  \{ v\in \mathbb{R}^{n} | \; A_1 (x_1 - v) = \dots = 
    A_{\ntransf} (x_{\ntransf} -v) = 0, \; x_1,\dots,x_{\ntransf}\in \signalset \}
\end{equation}
Even though we have derived the set $\hat{\signalset}$ from a purely geometrical argument, the constraints in \Cref{eq:inferred set} also offer a simple algebraic intuition: the inferred signal set consists of the points $v\in \R{n}$ which verify
the following system of equations 
\begin{equation} \label{eq: algo_x}
    \begin{bmatrix}
     A_1 \\ 
     \vdots \\
     A_{\ntransf} 
    \end{bmatrix} v =   \begin{bmatrix}
     A_1x_1 \\ 
     \vdots \\
     A_{\ntransf} x_{\ntransf}
    \end{bmatrix}.
\end{equation}
for all possible choices of $\ntransf$ points $x_1,\dots,x_{\ntransf}$ in $\signalset$.
In other words, given a dataset of $N$ incomplete samples $\{A_{g_i}x_i\}_{i=1}^{N}$, it is possible to build $\hat{\signalset}$ by trying all the possible combinations of $\ntransf$ samples\footnote{Despite providing a good intuition, this procedure for estimating $\signalset$ is far from being practical as it would require an infinite number of observed samples if the dimension of the signal set is not trivial $k>0$.} and keeping only the points $v$ which are the solutions of~\Cref{eq: algo_x}.

It is trivial to see that $\signalset\subseteq \hat{\signalset}$, but when can we guarantee $\signalset=\hat{\signalset}$? 
As in the previous toy example, if there are not enough constraints, e.g., if we have a single $A$ and no additional transformations, the inferred set will have a dimension larger than $k$, containing undesired aliases. In particular, we have the following \JT{necessary condition on the operators}:

\begin{proposition}[Theorem 1 in~\citep{chen2021equivariant}] \label{prop: necessary multA}
\JT{Let the operators $A_1,\dots,A_{\ntransf}\in\R{m\times n}$. The signal set $\signalset$ can be uniquely identified from the sets $\{A_g\signalset\}_{g=1}^{\ntransf}$ \JT{only} if \begin{equation}\label{eq: rank condition}
    \rk{\begin{bmatrix}
    A_{1} \\
    \vdots \\
    A_{\ntransf}
    \end{bmatrix}} = n
\end{equation} and thus \JT{only} if $m\geq n/\ntransf$.}
\end{proposition}
\begin{proof}
In order to have model uniqueness, the system in~\cref{eq: algo_x} should only admit a solution  if $v=x_1=\dots=x_{\ntransf}$. \JT{If the rank condition in~\cref{eq: rank condition} is not satisfied, there is more than one solution for any choice of $x_1,\dots,x_{\ntransf}\in\signalset$.} 
\end{proof}

Note that this necessary condition does not take into account the dimension of the model. As discussed in~\Cref{subsec: highdim}, a sufficient condition for model uniqueness must depend on the dimension of the signal set $k$. The next theorem shows that $k$ additional measurements per operator are sufficient for model identification:

\begin{theorem} \label{theo: multiple op}
\JT{Assume that the signal set $\signalset$ verifies assumption \textup{\textbf{A1}}}. For almost every set of $\ntransf$ operators $A_1,\dots,A_{\ntransf}\in \R{m\times n}$ \JT{w.r.t.\ the Lebesgue measure on $\mathbb{R}^{\ntransf n m}$,} the signal model $\signalset$ can be uniquely identified \JT{from the measurement sets $\{A_g\signalset\}_{g=1}^{\ntransf}$} if the number of measurements per operator  verifies $m> k + n/\ntransf$.
\end{theorem}

The proof is included in~\Cref{app: multop proof}. \JT{It is worth noting that almost every set of $\ntransf$ operators verifies the necessary condition in~\Cref{prop: necessary multA}}. 
If we have a large number of independent operators $\ntransf\geq n$, \Cref{theo: multiple op} states that only $m>k+1$ measurements are sufficient for model identification, which is slightly smaller (if the model is not trivial, i.e., $k>1$) than the number of measurements typically needed for signal recovery $m>2k$. In this case, it is possible to uniquely identify the model, without necessarily having a perfect reconstruction of each observed signal. 
However, as discussed in~\Cref{subsec:CS preliminary}, for $k<m\leq 2k$, the subset of signals which cannot be uniquely recovered is for almost all $A_g$ at most $(2k-m)$-dimensional.



\paragraph{Operators with Different Number of Measurements} \label{subsec: different m}
The results of the previous subsections can be easily extended to the setting where each measurement operator has a different number of measurements, i.e., $A_1\in \R{m_1 \times n},\dots, A_{\ntransf}\in \R{m_{\ntransf} \times n}$. In this case, the necessary condition in Proposition~\ref{prop: necessary multA} is $\sum_{g=1}^{\ntransf} m_g \geq n$, and the sufficient condition in~\Cref{theo: multiple op} is $\frac{1}{\ntransf}\sum_{g=1}^{\ntransf} m_g > k + n/\ntransf$.
As the proofs mirror the ones of Proposition~\ref{prop: necessary multA} and  \Cref{theo: multiple op}, we leave the details to the reader.

\section{Learning with a Single Operator and Group Symmetry} \label{sec: group learning}

We begin with some notation on groups and their associated linear representations.
A group $\group$ with multiplication denoted by $\cdot$ is a set of $|\group|$ elements $g$, such that the product of two elements is equal to another element in $\group$ ($g_1\cdot g_2 \in \group$), $\group$ contains the identity element $1_\group$ and the inverse of every element is contained in $\group$, i.e., $g^{-1}\in\group$ and $g\cdot g^{-1}=1_\group$.  The action of a group in a vector space $V$ (e.g., $\R{n}$ or $\C{n}$) can be represented by invertible transformations $T_g$ acting on $V$, where the $\cdot$ operation is just matrix multiplication, and the transformations are compatible with the group multiplication, i.e., $T_g T_{g'}$ = $T_{g\cdot g'}$. 

\JT{In this section, we assume that we observe the measurement set $\yset=A\signalset$ defined as}
\begin{equation}
    \JT{\yset = \{ y\in\R{m} |  \;y=Ax, \; x\in\signalset \}}.
\end{equation}
and that the signal set $\signalset$ is invariant to the action of a group $\group$, that is $T_g\signalset = \signalset$ for all transformations $T_1,\dots, T_{\ntransf} \in \R{n\times n}$.
As with \Cref{sec:multiple ops}, in practice we will only observe a set of $N$ training measurement data $y_i$ via a single operator, i.e.,
\begin{equation}
    y_i = A x_i
\end{equation}
for $i=1,\dots,N$. We do not assume that the distribution $p(x)$ is also $\group$-invariant, which is a more stringent condition, \red{although this setting is still covered here as a special case, since any $\group$-invariant $p(x)$ has a $\group$-invariant support.}
As with the multiple operator case, we first focus on the noiseless case, leaving the discussion about the effect of noise for~\Cref{sec: noise}.  

The invariance property provides us with a simple but powerful way of learning the model with incomplete measurements~\citep{chen2021equivariant}:
for all group elements the following holds
\begin{align} \label{eq: G operators}
    y_i &= Ax_i \\
      &= AT_g T_g^{-1}x \\
      &= A_g x_{i'} \label{eq:virtual}
\end{align}
where both $x_i$ and $x_{i'}$ belong to the signal set due to the invariance property. Equation~\Cref{eq:virtual} tells us that we can implicitly access $\ntransf$ different measurement operators $A_g=AT_g$, each one associated with a nullspace $\nullAg$ which might differ from $\nullA$. Thus, we have a similar setting
as in~\Cref{sec:multiple ops}, with the additional constraint that the operators $A_g$ are related by some transform $A_g=AT_g$, where the transforms $\{T_g\}_{g\in\group}$ can be seen as the action of some group $\group$.
As the virtual forward operators $A_g$ are related by the group action, we cannot directly apply the results in~\Cref{sec:multiple ops}, which assume that the operators are independent, and we require further analysis. 
The following toy example provides intuition of how the invariance property can help in learning the model:

\begin{example}[Shift invariant toy model]\label{ex: 2D toyb}
Consider again the one-dimensional subspace model presented in \Cref{ex: 2D toy}. We now 
assume that the model is shift invariant, including all shifts of every signal, but we only partially observe signals through a fixed masking operator. While the original non-invariant set $\baseset$ is a one-dimensional linear subspace, the full set $\signalset$ is now a union of linear one-dimensional subspaces $\bigcup_{g\in\group}T_g\baseset$. While $\signalset$ is much larger than $\baseset$, the intrinsic degrees of freedom are the same, i.e., the triangular wave $\phi$. Despite having a single mask, we can access to multiple virtual operators $AT_g$ via \Cref{eq:virtual}, as illustrated in~\Cref{fig:toy invariant}. Solving both model identification and signal recovery problem reduces to high rank (union of subspaces) matrix completion~\citep{eriksson2012high} of the measurement matrix $[y_1,\dots,y_N]$.
\end{example}
\begin{figure}[h]
     \centering
     \includegraphics[width=.5\textwidth]{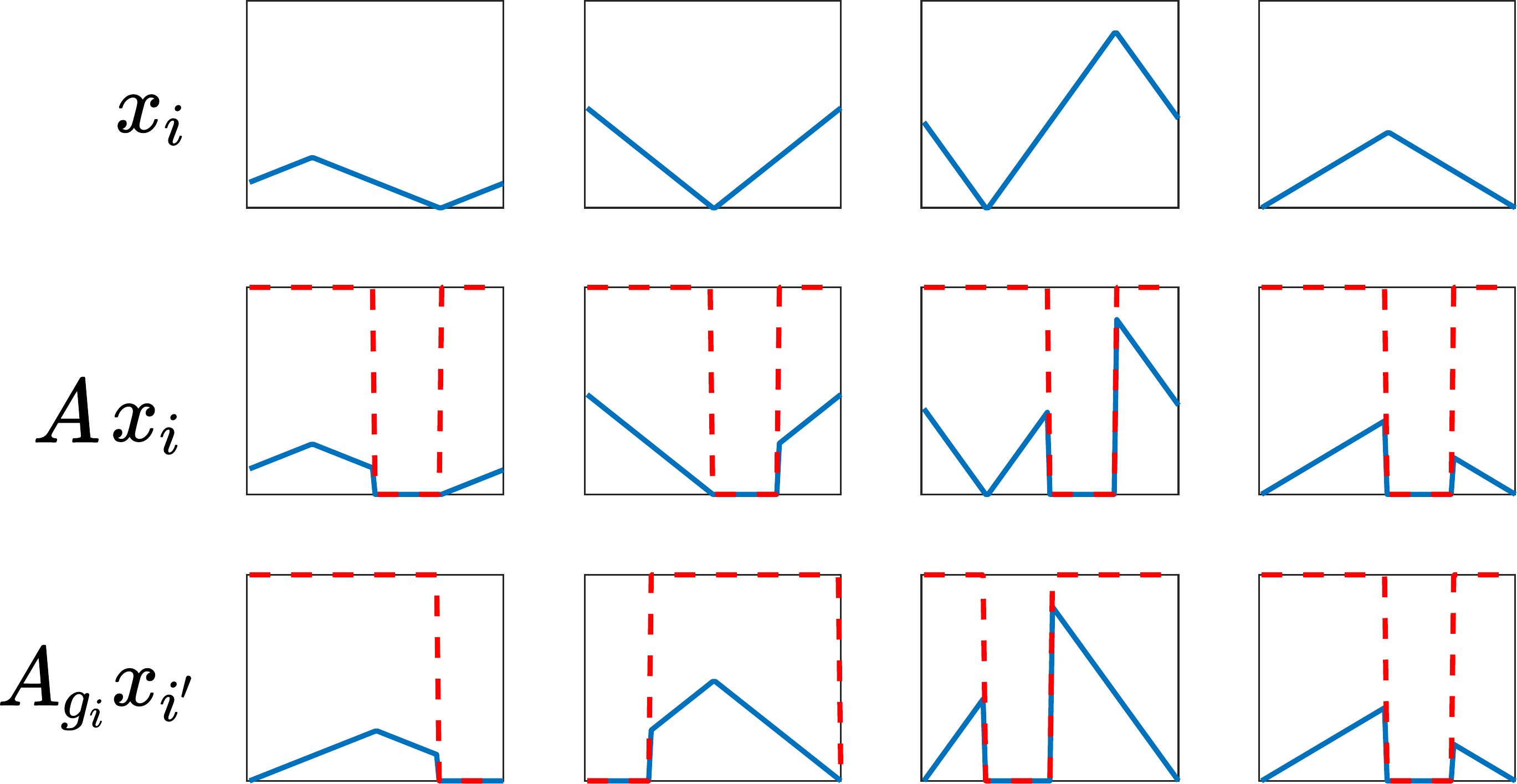}
\caption{Learning a one-dimensional shift invariant model from masked \JT{observations}. Each column shows a different \JT{observation} from the signal set. The \JT{sensing} mask is shown in red. Following~\cref{eq:virtual}, the samples can be reinterpreted as  having measured another valid signal $x_{i'}$ through a different operator $A_{g_i}$.}\label{fig:toy invariant}
\end{figure}

Before delving into the main results of this section, we briefly introduce some fundamental concepts of linear representation theory and discuss some common group actions.

\subsection{Linear Representation Theory Preliminaries} \label{subsec: group theory and examples}

A classical result in linear representation theory states that $\C{n}$ admits a unique decomposition into invariant\footnote{An invariant subspace is a subspace of $\C{n}$ which is invariant to all transformations $\{T_g\}_{g\in\group}$.} subspaces~\citep{serre1977linear}:

\begin{theorem}[Canonical decomposition of a representation, adapted from~\citep{serre1977linear}] \label{theo: irreps}
Let $\group$ be a \red{compact} group acting on $\C{n}$, whose action is represented by the invertible linear mappings $T_g$. There is a unique canonical decomposition of $\C{n}$ into $J$ invariant subspaces as
\begin{equation} \label{eq: canonical decomp}
    \C{n} = V_1 \oplus \dots \oplus V_J 
\end{equation}
where each subspace has dimension $s_jc_j$ with $s_j$ the dimension of the $j$th irreducible representation and $c_j$ its multiplicity. Moreover, any linear representation admits a block-diagonal form \begin{equation}
    T_g = F \Sigma_g F^{-1}
\end{equation}
where $F$ is a basis for~\Cref{eq: canonical decomp} and
\begin{equation}
    \Sigma_g = \left[ \begin{array}{ccc}
\Lambda_1(g) &  &   \\
 & \ddots &   \\
 &   & \Lambda_J(g) \\
\end{array} \right]
\; \text{with} \; 
   \Lambda_j(g) = \left[ \begin{array}{ccc}
\rho_j(g) &  &   \\
 & \ddots &   \\
 &   & \rho_j(g) \\
\end{array} \right] \in \mathbb{C}^{c_js_j\times c_js_j} \; \text{and} \; \rho_j(g)\in \mathbb{C}^{s_j\times s_j}
\end{equation}
The matrices $\rho_j(g)$ are unique (independent of the basis $F$) and correspond to the $j$th irreducible representation of $\group$.
\end{theorem}

In words, a group is linked to a unique set of $J$ irreducible representations of dimension $s_1,\dots,s_J$, which serve as fundamental building blocks of any linear group action. A linear group action on $\C{n}$ consists of a specific basis $F$ and a set of multiplicities $c_j$  such that $\sum_{j=1}^{J} s_jc_j = n$. The vector space on which the group acts admits a decomposition into $J$ invariant subspaces of dimension $s_1c_1,\dots,s_Jc_J$, each associated with a different irreducible representation.
To build some intuition about~\Cref{theo: irreps}, we discuss its  implications for common group actions such as shifts, reflection, rotations and permutations on 1D (e.g., audio) or  2D discretely sampled signals (e.g., images):

\subsubsection{Finite Groups} \label{subsubsec: finite groups}

\paragraph{Shifts} The cyclic group of $n$ elements associated with circulant matrices $T_{g_i}$ which shift the signal by $i$ taps/pixels, with $i\in\{0,\dots,n-1\}$. In this case, we have $\ntransf=J=n$ and $c_1=\dots=c_J=1$ and $s_j=1$, where the subspaces are the Fourier modes. 
Here, $F$ is the discrete Fourier transform and $\Lambda$ is a diagonal matrix containing the Fourier transform of the discrete shift operation.

\paragraph{Reflection} The cyclic group of 2 elements, the identity and a reflection of a 1D signal about its central entry ($\ntransf=2$).  For even\footnote{The extension to odd $n$ is trivial.} $n$, we have $J=2$, $s_1=1$ with $c_1=n/2$ composed of vectors 
$[v_1,\dots,v_{n/2},v_{n/2},\dots,v_1]^{\top}$ with $v\in\C{n/2}$, and $s_2=1$ with $c_2=n/2$, composed of vectors $[v_1,\dots,v_{n/2},-v_{n/2},\dots,-v_1]^{\top}$. The matrix $F$ can be built as 
\begin{equation}
    \frac{1}{\sqrt{2}}\left[ \begin{array}{cc}
    \tilde{F} &  \tilde{F} \\
     T_1\tilde{F} &  -T_1 \tilde{F} \\
    \end{array} \right]
\end{equation}
where $T_1\in \R{n/2\times n/2}$ is a
reflection matrix with entries $[T_1]_{i,j}=1$ if $i+j=n/2+1$ and $0$ otherwise, and $\tilde{F}\in \mathbb{C}^{n/2\times n/2}$ is any unitary matrix.

\paragraph{Rotations} To simplify the exposition, here we  consider a 2D circular \JT{sensing} pattern with $n_1$ pixels of diameter and angle sampled every $360/n_2$ degrees (hence $n=n_1n_2$). The cyclic group $\group$ is then represented as the set of all $\ntransf=n_2$ possible rotations.  In this case, we have $J=n_2$, $s_1=\dots=s_J=1$ and $c_1=\dots=c_J=n_1$.  The matrix $F$ is given by the Kronecker product between a 1D discrete Fourier transform along the angular dimension $F_1\in \mathbb{C}^{n_2\times n_2}$ and any unitary matrix along the radial dimension $F_2\in \mathbb{C}^{n_1\times n_1}$, i.e.,  $F_1\otimes F_2$.

\paragraph{Permutations} Full set of permutations of the $n$ entries of the signal with $\ntransf=n!$ elements. This group can model permutations of a masking operator, exchangeable distributions~\citep{eaton1989group} or \red{patches in an image}. \JT{Commonly-used separable distributions~\citep{guo2015near} are also permutation invariant.}. This is a significantly larger group (it contains all the rest as subgroups), where $J=2$, $s_1=c_1=1$ for the subspace spanned by $[1,\dots,1]^{\top}$ and its complement $s_2=n-1$ with $c_2=1$. Here, $F$ is given by any unitary matrix whose first column is the constant vector $[\frac{1}{\sqrt{n}},\dots,\frac{1}{\sqrt{n}}]^{\top}$.

\subsubsection{Infinite Groups} \label{subsubsec: infinite groups}

\red{ Natural signals are often described as continuous functions $x:S\mapsto \R{p}$, where $S$ is a compact domain and $\R{p}$ is the range. For example, an RGB image can be defined using $S=(0,1]^2$ and $p=3$. In this description, the ambient space is generally an infinite-dimensional Hilbert space $\Hilb$  and  group actions are generally continuous, such as translations or arbitrary rotations. }

\red{Formally, the linear representation of a group acting on $\Hilb$ is given by a mapping from $\group$ to the space of unitary linear mappings $T_g:\Hilb \mapsto\Hilb$.
The decomposition into invariant subspaces of~\Cref{theo: irreps} also holds in this setting, where the number of irreducible representations with non-zero multiplicity $J$ can be infinite~\citep[Chapter~4]{serre1977linear}.}

\red{In practice, we generally assume that the signals of interest are bandlimited and, according to Nyquist theorem, can be represented on a discrete grid associated with $\R{n}$. 
Restricted to the bandlimited subspace, the group action\footnote{\red{If an antialising filter is used, the discretization process is an equivariant mapping, i.e., there is a valid representation of $\group$ acting on $\R{n}$}.} has a finite number of non-zero multiplicities $c_j$ and \Cref{theo: irreps} holds despite the infinite number of group elements ($\ntransf=\infty$). 
As Nyquist theorem guarantees a one-to-one mapping between the \JT{signal model in $\R{n}$} and the signal model defined in $\Hilb$, we only need to prove model identifiability using signals and (possibly continuous) group actions on $\R{n}$ in order to prove model identifiability in the continuous case. We illustrate these concepts with two simple examples:}

\paragraph{Translations}
\red{
Consider the space $L_2(0,1]$ of square integrable signals taking values in the $(0,1]$ interval where the group of translations is represented via the left action
\begin{equation}
    T_gx = x(t-g)
\end{equation}
with $g\in(0,1]$, where $t-g$ is modulo 1. This representation has $s_j=c_j=1$ for all $j\in \mathbb{Z}$, where each invariant subspace corresponds to the complex exponential $e^{-\mathrm{i}2\pi jt}$. Assuming that the signals of interest have a maximum frequency $J_{\max}\in \mathbb{N}$, we can represent them on a grid as a vector $\tilde{x}\in\R{n}$, $n=2J_{\max}+1$, with entries given by
\begin{equation}
    \tilde{x}(r)= \int_0^1 x(t) k(t-\frac{r}{n}) dt 
\end{equation}
for $r=1,\dots,n$ where $k(t)=\frac{\sin(\pi tn)}{\sin (\pi t)}$ is the periodic sinc kernel (the ideal\footnote{In practice, images are usually sampled on a grid and non ideal anti-aliasing filters are used hence there is only approximate equivariance.} anti-aliasing filter). On $\R{n}$, the group action is given by diagonal matrices
\begin{equation}
    \tilde{T}_g = F\begin{bmatrix}
    e^{-\mathrm{i}2\pi j_1 g} & & \\
    & \ddots & \\
    & & e^{-\mathrm{i}2\pi j_n g}
    \end{bmatrix}F^{-1}
\end{equation}
where $j_1<\dots<j_n$ are the integers in $[-J_{\max},J_{\max}]$, $F$ is the discrete Fourier transform  and $g\in(0,1]$. This representation has $s_j=c_j=1$ for all $j=1,\dots,n$. It is worth noting that, restricted to a subgroup of $n$ equispaced elements, this group action is equivalent to the discrete shifts example in~\Cref{subsubsec: finite groups}.
}

\paragraph{Continuous Rotations}
\red{Consider the set of continuous signals whose domain $S$ is the unit circle and the range is $\R{}$. The signals can be written as $x(r,\theta)$ with radius $r\in (0,1]$ and angle $\theta\in (0,2\pi]$. The left action of  the group of rotations is given by
\begin{equation}
    T_gx = x(r,\theta-g)
\end{equation}
with $g\in(0,2\pi]$, where $\theta-g$ is modulo $2\pi$. Assuming that the signals of interest have a maximum frequency $J_{\max_1}$, $n_1=2J_{\max_1}+1$, on the radial axis and a maximum frequency $J_{\max_2}$, $n_2=2J_{\max_2}+1$, on the angular axis, we can represent them as vectors $\tilde{x}\in\R{n_1 n_2}$ on a discrete grid by  the convolution with the periodic sinc kernel every $\Delta r=1/n_1$ and $\Delta\theta=2\pi/n_2$. The action of rotations in $\R{n_1n_2}$ has $s_1=\dots=s_J=1$ and $c_1=\dots=c_J=n_1$. Restricted to the subgroup of rotations by $360/n_2$ degrees, the representation  is equivalent to the discrete rotations example in~\Cref{subsubsec: finite groups}.
}

\subsection{Uniqueness of Any $\group$-Invariant Model?} \label{subsec: high dim G models}


As with the case of multiple independent operators, we first ask 
the question of whether it is possible to uniquely identify any $\group$-invariant model from measurements of a single operator with $m<n$, regardless of the dimension of its support $\signalset$. We show that even under the assumption that $p(x)$ is $\group$-invariant, uniqueness is not possible if the support has dimension close to the ambient dimension.
If $p(x)$ is $\group$-invariant, all points in the same orbit,  $\mathcal{O}_x$, defined as
\begin{equation}
    \mathcal{O}_x = \{T_gx | \;x\in \C{n}, \; \forall g\in\group \}
\end{equation}
will necessarily have the same probability. Hence, the distribution is uniquely identified by the distribution over the orbits $\mathcal{O}_x$, instead of the whole space $\C{n}$. This notion is made precise by a fundamental theorem of invariant statistics:

\begin{theorem}[Decomposable measures, adapted from~\citep{eaton1989group}] A $\group$-invariant distribution on $\C{n}$ can be decomposed into a uniform distribution over the elements in $\group$ and a distribution over the quotient space $\C{n}/\group$.
\end{theorem} 

The theorem tells us that we can potentially design a measurement operator with $m<n$ which is one-to-one with respect to the quotient space (having model uniqueness) but not necessarily one-to-one with respect to the signal space (failing to have signal recovery). For infinite compact groups, the quotient space might be significantly smaller than the full signal space.  For example, consider distributions which are invariant to the action of the orthogonal group $O_n$ composed of all orthogonal matrices. These are necessarily isotropic distributions only dependent on the radius  $\|x\|$. In this case, the full space is $n$-dimensional, but the quotient space is only one-dimensional (the sufficient statistics are functions of $\|x\|$), and (non-linear) one-dimensional measurements $A(x) = \|x\|$ are sufficient to identify any $\group$-invariant distribution. 
On the other hand, the orbits of finite groups\footnote{A similar statement holds for smaller infinite and compact groups, such as continuous translations of 1D or 2D signals.}, which are the main focus of this paper, are sets consisting of a finite number of signals, and the quotient space is intrinsically $n$-dimensional:

\begin{theorem} [Dimension of quotient space~\citep{sturmfels2008algorithms}]
The quotient space $\mathbb{C}^n/ \group$ of any finite group acting on $\mathbb{C}^n$ is $n$-dimensional.
\end{theorem}
 As $\mathbb{C}^n/ \group$ is $n$-dimensional for any finite group, any forward operator $A\in \mathbb{C}^{m\times n}$ with $m<n$ cannot be one-to-one on $\mathbb{C}^n/\group$, hence some invariant distributions will not be unique.
Despite not being able to learn any $\group$-invariant model, we show that, as in~\Cref{sec:multiple ops}, uniqueness among low-dimensional $\group$-invariant models is indeed possible.

\subsection{Low-Dimensional $\group$-Invariant Models} \label{subsec: lowdim G models}

The observation in~\Cref{eq:virtual} tells us that $\group$-invariant models provide access to $\ntransf$ (virtual) operators $AT_g$ with $g\in \group$. Thus, we can apply similar ideas to those used in~\Cref{theo: multiple op} to analyze the problem of learning a $\group$-invariant model, with the additional constraint that the operators are $\group$-related. 
Instead of depending on the ratio between ambient dimension $n$ and the number of transformations $\ntransf$, the bounds depend on the largest ratio of multiplicity and dimension of the irreducible components of the group action (as defined in~\Cref{theo: irreps}), that is
\begin{equation}
    \max_{j} \frac{c_j}{s_j}.
\end{equation}
With this definition in mind, we begin by reformulating the necessary condition in Proposition~\ref{prop: necessary multA}:

\begin{proposition} \label{prop: necessary G structure}
\red{Let $\group$ be a compact group. A necessary condition for model uniqueness with  operators $\{AT_g\}_{g\in\group}$ is $m \geq \max_{j} \frac{c_j}{s_j}$.}
\end{proposition}
The proof is detailed in \Cref{app: group proofs}, and similarly to the proof of Proposition~\ref{prop: necessary multA}, requires analyzing the rank of the matrix (or its infinite dimensional equivalent for infinite groups)
\begin{align}
      \begin{bmatrix}
     A_1 \\ 
     \vdots \\
     A_{\ntransf} 
    \end{bmatrix}
\end{align}
with the additional constraint that $A_g = AT_g$. Note that for any group action $n/\ntransf\leq \max_j c_j/s_j \leq n$, hence this is always a more restrictive bound than the one in Proposition~\ref{prop: necessary multA}. The following two  examples present group actions where $\max_j c_j/s_j \gg n/\ntransf$:

\begin{example}
    Consider the group of permutations of the last $4$ elements of a vector of $10$ elements where $\ntransf=4!=24$. The group action has $\max_j c_j/s_j = 6$ associated with the subspace containing the first $6$ elements, whereas $n/\ntransf<1$. 
\end{example}

\begin{example}
    \red{Consider the infinite group of rotations of a signal introduced in~\Cref{subsubsec: infinite groups}. The group action has $\max_j c_j/s_j = n_2$, however $\ntransf=\infty$.}
\end{example}

We now study a sufficient condition for unique model identification. As we have that $\max_j c_j/s_j \geq n/\ntransf$, one might be tempted to extrapolate the bound in~\Cref{theo: multiple op}, and expect that $m>k + \max_j c_j/s_j  $ measurements are sufficient for model identification. However, the following counter-example shows that this is not always sufficient:

\begin{example}\label{ex: worst case ref}
Let $\group$ be the group of reflection with 2 elements $\{e,r\}$ of a signal with even length $n$ \red{(see~\Cref{subsubsec: finite groups})}. Consider the representation given by $T_e= I_n$ where $I_n$ is the $n\times n$ identity matrix and
\begin{equation}
    T_r = \begin{bmatrix}
    I_{n/2} & 0 \\
    0 & -I_{n/2}
    \end{bmatrix}.
\end{equation}
Note that the representation has $J=2$ invariant subspaces with $s_1=s_2=1$ and $c_1=c_2=n/2$. Let $\signalset$ be the set of $d$-sparse signals where $n> 2d$. This conic set has dimension $k=d-1$ according to the definition in \Cref{subsec: low-dim models}. In order to have model uniqueness, we require that the inferred signal set $\hat{\signalset}$ defined in \Cref{eq:inferred set} equals the true set $\signalset$, or equivalently that their difference 
\begin{equation} \label{eq:diff inf example}
    \hat{\signalset}\setminus{\signalset} = \{ v\in \R{n}\setminus{\signalset} | \;  A T_e (x_1 - v) = 
    A  T_r(x_{2} -v) = 0, \; x_1,x_{2}\in \signalset \}
\end{equation}
is empty. This condition can be written in matrix form as
\begin{equation} \label{eq: counterex condition}
    A[\tilde{x}_1-T_ev, \tilde{x}_2-T_rv] \neq 0 .
\end{equation}
where $\tilde{x}_1 = T_e x_1$ and $\tilde{x}_2 = T_rx_2$ are $d$-sparse and thus also belong to the signal set $\signalset$.
Let $\Phi_z = [\tilde{x}_1-T_ev, \tilde{x}_2-T_rv]$, such that condition~\cref{eq: counterex condition} can be written as $A\Phi_z \neq 0$. We consider the (extreme) case where $\tilde{x}_1$ and $\tilde{x}_2$ have disjoint supports, i.e.
$\text{supp}(\tilde{x}_1) = [n/2+1,...n/2+d]$
and 
$\text{supp}(\tilde{x}_2) = [n/2+d+1,...n/2+2d]$. Using coordinates $[a_1,\dots,a_d]^{\top} \in \R{d}$ for $\tilde{x}_1$ and $[b_1,\dots,b_d]^{\top} \in \R{d}$ for $\tilde{x}_2$, we obtain
\begin{equation}
  [\tilde{x}_1-T_ev, \tilde{x}_2-T_rv] =  \begin{bmatrix}
    -v_1 & -v_1 \\
    \vdots & \vdots \\
    -v_{n/2} & -v_{n/2} \\
    a_1 -v_{n/2+1} & v_{n/2+1} \\
    \vdots & \vdots \\
    a_d - v_{n/2+d} & v_{n/2+d} \\
     - v_{n/2+d+1} & b_1+v_{n/2+d+1} \\
    \vdots & \vdots \\
     - v_{n/2+2d} & b_{d}+v_{n/2+2d} \\
    \vdots & \vdots \\
    -v_n & v_n
    \end{bmatrix}
\end{equation}
Setting $a_i=2v_{n/2+i}$, $b_i=-2v_{n/2+d+i}$ for $i=1,\dots,d$ and $v_{j}=0$ for  $j=n/2+2d+1,\dots,n$, we have that
$[\tilde{x}_1-T_ev, \tilde{x}_2-T_rv] = [\mu, \mu]$ with $\mu = -[v_1,\dots,v_{n/2+2d},0,\dots,0]^{\top}$. Hence, condition \cref{eq: counterex condition} can be simplified to
\begin{align}\label{eq: restrA cond}
    \tilde{A} \begin{bmatrix}
    v_1 \\
    \vdots \\
    v_{n/2+2d}
    \end{bmatrix} \neq 0
\end{align}
where $\tilde{A}$ is a $m\times (n/2+2d)$ submatrix of $A$. Thus, the condition in \cref{eq: restrA cond} can only hold for all $[v_1,\dots,v_{n/2+2d}]^{\top}\in \R{n/2+2d}-\{0\}$ if $m\geq n/2+2d$. As this linear representation has $\max_j c_j/s_j = n/2$, and $k=d-1$, we have 
\begin{equation}
    m\geq \max_j c_j/s_j + 2k - 2.
\end{equation}
\end{example}

Our main theorem shows that, for a cyclic group, $2k+1$ additional  measurements are actually sufficient to guarantee model uniqueness for almost every $A$:  

\begin{theorem} \label{theo: one op}
Let $\group$ be a compact cyclic group acting on $\R{n}$ \JT{and assume the signal set~$\signalset$ verifies assumption \textup{\textbf{A1}}}. For almost every $A\in\R{m\times n}$ \JT{w.r.t.\ to the Lebesgue measure in $\R{mn}$,} the signal set can be uniquely identified \JT{from the sets $\{AT_g\signalset\}_{g\in\group}$} if the number of measurements verifies \red{$m > 2k + \max_j c_j + 1$}. In particular, for almost any operator $A\in \mathbb{R}^{m\times n}$, a $\group$-invariant signal set can be identified from \JT{the measurement set $\yset=A\signalset$} if \red{$m > 2k + \max_j c_j +1$}.
\end{theorem}

The proof is detailed in \Cref{app: group proofs}. As all the irreducibles of a cyclic group have dimension $s_j=1$~\citep{serre1977linear}, the bound can also be expressed as $m > 2k + \max_j c_j/s_j +1$.
Unlike the case of multiple independent operators, having enough measurements to guarantee model identification also guarantees unique signal recovery due to the additional $k$ measurements. However, the experiments in~\Cref{sec:experiments} suggest that $m>2k+\max_j c_j/s_j$ measurements are only required in worst-case scenarios such as~\Cref{ex: worst case ref}, whereas $m>k+\max_j c_j/s_j$ are required for typical models.

Although~\Cref{theo: one op} is specific to actions of cyclic groups, it can be applied to any group action, as all groups contain at least one cyclic subgroup. For example, if the group action consists of the composition of rotations and reflection, which is not cyclic nor abelian, we can apply~\Cref{theo: one op} by only considering rotations or reflection, which are cyclic when considered separately. However, if the cyclic subgroup is significantly smaller than the full group, \Cref{theo: one op} might not be tight. We conjecture that the bound in~\Cref{theo: one op} holds for any compact group:

\begin{conjecture}\label{conj: all groups}
Let $\group$ be any compact group acting on $\R{n}$ \JT{and assume the signal set~$\signalset$ verifies assumption \textup{\textbf{A1}}}. For almost every $A\in\R{m\times n}$ \JT{w.r.t.\ to the Lebesgue measure in $\R{mn}$,} the signal set can be uniquely identified \JT{from the sets $\{AT_g\signalset\}_{g\in\group}$} if the number of measurements verifies \red{$m > 2k + \max_j c_j/s_j + 1$}. In particular, for almost any operator $A\in \mathbb{R}^{m\times n}$, a $\group$-invariant signal set can be identified from \JT{the measurement set $\yset=A\signalset$} if \red{$m > 2k + \max_j c_j/s_j +1$}.
\end{conjecture}

The \emph{almost every} property can be hard to interpret in practical scenarios. A simple example of matrices that will hold this property with probability 1 are compressive sensing matrices (e.g., iid random Gaussian entries). However, there is an important subset of measurement operators of Lebesgue measure zero in  $\mathbb{R}^{m\times n}$ that do not verify this property and are treated in the next subsection. 

\subsection{Equivariant Measurement Operators}
In order to obtain additional information in the nullspace of $A$, Proposition~\ref{prop: necessary multA} requires that the measurements operators $AT_1,\dots,AT_{\ntransf}$ do not share the same nullspace. 
Operators which fail to bring information in the nullspace of $A$ possess the property of being $\group$-equivariant:

\begin{theorem} \label{theo: badAs}
The set of forward operators $\{AT_g\}_{g\in\group}$ share the same nullspace for all $g\in \group$ if and only if $A$ is an equivariant map, i.e., it verifies
\begin{equation}\label{eq:equivariant map}
    A T_g = \tilde{T}_g A
\end{equation}
for all $g\in\group$, where $\tilde{T}_g:\group\mapsto\C{m\times m}$ is a linear representation of $\group$ acting on $\C{m}$. Moreover, any equivariant linear operator can be decomposed as $A= \tilde{F} \Lambda F^{-1}$, where $F$ is the basis associated to the group representation in \Cref{theo: irreps},  $\tilde{F}$ is any basis of $\C{m}$ and $\Lambda\in\C{m\times n}$ has the following block-diagonal structure
\begin{equation}
\label{eq:schur}
    \Lambda = \left[ \begin{array}{ccc}
\Lambda_{1} &  &   \\
 & \ddots &   \\
 &   & \Lambda_{J} \\
\end{array} \right]
\; \text{with} \; 
   \Lambda_{j} = \left[ \begin{array}{ccc}
B_{j} &  &   \\
 & \ddots &   \\
 &   & B_{j} \\
\end{array} \right] \in \mathbb{C}^{\tilde{c}_{j}s_{j}\times c_{j}s_{j}} \; \text{and} \; B_j\in \mathbb{C}^{\tilde{c}_{j}\times c_{j}}
\end{equation}
where $\tilde{c}_j$ are the multiplicities of the representation $\tilde{T}_g$, such that $m=\sum_{j=1}^{J} s_{j} \tilde{c}_{j}$.
\end{theorem}
The proof is detailed in \Cref{app: group proofs}. This necessary condition on $A$ has important practical implications, as equivariant operators appear often in real-world settings which are discussed in next subsection. 

\begin{figure}[t]
\centering
\includegraphics[width=.8\textwidth]{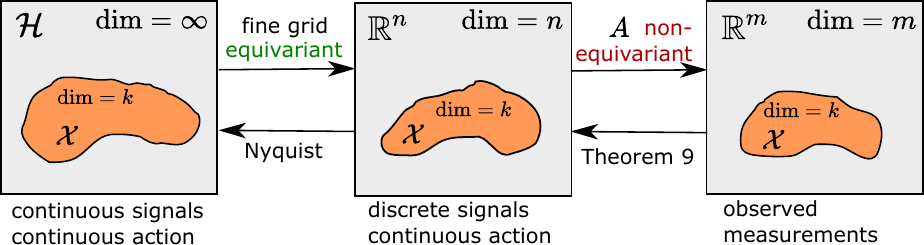}
\caption{Identification of models with continuous signals and \JT{continuous} group actions. Continuous signals are well described as points in an infinite dimensional Hilbert space $\Hilb$, and symmetries in the model are described as a continuous group acting on $\Hilb$. As discussed in~\Cref{subsubsec: infinite groups}, if the signals are bandlimited, we can represent them on a discrete grid $\R{n}$. Nyquist theorem ensures this mapping is equivariant and one-to-one, thus we have a valid group action on $\R{n}$. \Cref{theo: one op} guarantees that we can uniquely identify the model from measurement data alone via almost every (thus non-equivariant) mapping $A$ with $m>2k+\max_j c_j+1$ measurements if $\ntransf$ is a cyclic group.}
\label{fig:cont diagram}
\end{figure}
\red{We end this subsection with a remark concerning  continuous signals and groups actions (see~\Cref{fig:cont diagram}) defined in an infinite-dimensional Hilbert space. As discussed in~\Cref{subsubsec: infinite groups}, bandlimited signals $x$ can be represented in $\R{n}$ on a discrete grid after an anti-aliasing filter. This procedure is an equivariant mapping, as we have a valid linear representation of the group acting on $\R{n}$. Contrary to the grid \JT{sensing} step, the measurement process $A$ should be non-equivariant in order to be able to uniquely identify the model from measurement data alone.}

\subsection{Consequences for Common Group Actions}\label{subsec:group actions}
We discuss the implications of \Cref{theo: one op} and \Cref{theo: badAs} for the group actions introduced in~\Cref{subsec: group theory and examples}.

\paragraph{Translations/Shifts}  \red{Continuous translations and shifts have the irreducible representations with non-zero multiplicities.} A generic sufficient condition for model identification in \red{both} cases is $m> 2k + 2$.
Unique model identification is impossible with measurement operators consisting of a subset of Fourier basis vectors which are shift equivariant. This is the case of multiple inverse problems such as deblurring, super-resolution, accelerated MRI and limited angle CT. However, it is possible to uniquely identify the signal model using translation invariance in other inverse problems such as image inpainting or compressed sensing.

\paragraph{Reflection} As this is a small group, in order to guarantee model identification our theory requires $m> 2k+n/2+1$ measurements, significantly more than those needed for signal recovery. 

\paragraph{Rotations}
\red{If we consider signals defined on the unit circle and a 2D radial \JT{sensing} pattern} with $n_1$ pixels of diameter and angle sampled every $360/n_2$ degrees (hence $n=n_1n_2$), we have $\max_j c_j/s_j = n_1$ and thus need $m>2k+n_1+1$ for model identification, which is also more measurements than required for signal recovery.

General rotations of a Cartesian grid require some degree of interpolation for off-of-grid values. Thus, this set of transformations do not form an exact group action in general. However, if the signals are sufficiently band-limited and supported within a radius from the center, the Cartesian grid can be approximated by the radial grid, and thus we need $m>2k+n_1+1$ measurements.

Rotations are useful for learning the signal model in inverse problems which depend on Fourier measurements, as long as these measurements are not of circular stripes of Fourier space, as illustrated in~\Cref{fig:fourier pat}. Two important examples where learning is possible are limited angle CT~\citep{chen2021equivariant} and accelerated MRI~\citep{chen2021robust}, as the Fourier patterns are generally not rotationally invariant.

\begin{figure}[t]
\centering
\includegraphics[width=.7\textwidth]{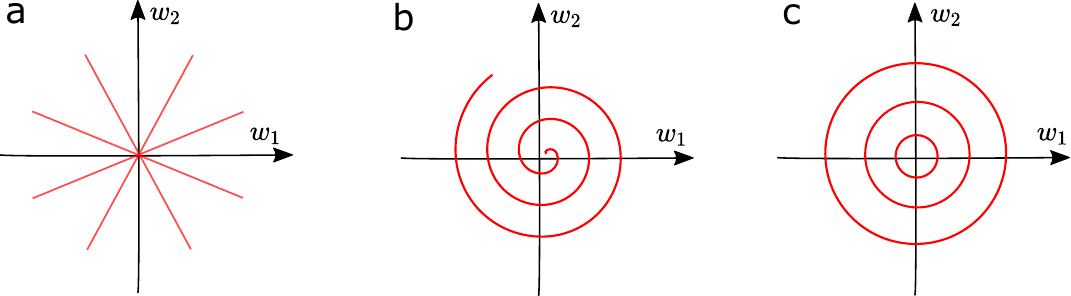}
\caption{Forward operators consisting of partial Fourier measurements. (\textbf{a}) The Radon transform appearing in tomography problems consists of lines in Fourier space. (\textbf{b}) Spiral frequency strategies are commonly used in MRI. (\textbf{c}) Circular Fourier pattern of measurements. 
Shift invariance does not provide any information in the nullspace. Rotation invariance provides additional information in both (\textbf{a}) and (\textbf{b}) but does not provide additional information in (\textbf{c}).}
\label{fig:fourier pat}
\end{figure}

\paragraph{Rotations by 90 degrees and Reflection}
 Consider the non-commutative group  of 8 elements whose action is composed of 90 degree rotations and reflection of a 2D signal defined on a regular grid. This group action has $\max_j c_j/s_j \approx n/8 $ for large $n$. \red{We conjecture that $m>2k+n/8+1$ measurements are sufficient for model identification}, significantly more than the case of rotations of a radial grid.

\paragraph{Permutations} Group of all permutations of $n$ entries which has $\max_j c_j/s_j=1$. Note that the previous examples are subgroups of this group. In this case, signal recovery also guarantees model identification.  The only equivariant measurement operator which does not provide additional information is the constant vector $A = \alpha[1,\dots,1]^{\top}$ with $\alpha\in\mathbb{R}$.

\section{Noisy Measurement Data} \label{sec: noise}
Surprisingly, the results of this paper are also theoretically valid if the measurements are corrupted by independent additive noise $\epsilon$, i.e., $y = A_{g}x + \epsilon$, as long as the noise distribution is \emph{known} and has a nowhere zero characteristic function (e.g., Gaussian noise):
\begin{proposition} \label{prop:noise}
For a fixed noise distribution, if its characteristic function is nowhere zero, then there is a one-to-one mapping between the space of clean measurement distributions and noise measurement distributions. 
\end{proposition}
\begin{proof}
Consider the noisy measurements associated to the $g$th operator $A_g$, as $z=y+\epsilon$, where $z$ are the observed noisy measurements, $y$ are the clean measurements and $\epsilon$ is additive noise (independent of $y$). 
The characteristic function of the sum of two independent random variables is given by the multiplication of their characteristic functions, i.e.,
\begin{equation}
    \varphi_z(w) = \varphi_y(w) \varphi_\epsilon(w)
\end{equation}
where $\varphi_z$, $\varphi_y$ and $\varphi_\epsilon$ are the characteristic functions the noisy measurement, clean measurements and noise distributions, respectively. If the characteristic function of the noise distribution is nowhere zero, we can uniquely identify the characteristic function of the clean measurement distribution as 
\begin{equation}
 \varphi_y(w)   = \varphi_z(w)/ \varphi_\epsilon(w)
\end{equation}
The clean measurement distribution is fully characterized by its characteristic function $\varphi_y(w)$. We end the proof by noting that the same reasoning applies to the measurements of every operator $A_g$ with $g\in \{1,\dots,\ntransf\}$.
\end{proof} 

If the clean measurement distribution can be uniquely identified, we can then apply \Cref{theo: multiple op} and Proposition~\ref{prop: necessary multA} for the case of independent operators, and \Cref{theo: one op,theo: badAs} and Proposition~\ref{prop: necessary G structure} for the case of $\group$-related operators. Note that this only guarantees model identifiability and makes no claims on the sample complexity of any learning process.

\section{Algorithms}
In this section, we summarize some different approaches for solving  model identification and/or signal recovery problems. \red{We recall the learning settings:
\begin{description}
    \item[Multiple operators] We observe $N$ measurement vectors $y_i$, each associated with one out of $\ntransf$ operators, $y_i = A_{g_i}x_i+\epsilon_i$, where $A_{g_i}$ is known.
    \item[Single operator, $\group$-invariant model] We observe $N$ measurement vectors $y_i$, associated with a single (known) operator, $y_i = Ax_i+\epsilon_i$, and assume that the set of signals $\signalset$ is $\group$-invariant.
\end{description}
}

\subsection{Learning a Generative Model}
One approach is to first learn the low-dimensional support $\signalset$ and then reconstruct the signals. This can be done in practice by estimating $\hat{\signalset}$ within a class of low-dimensional models $\mathcal{C}$ while satisfying measurement consistency:
\begin{align}
    \argmin_{\hat{\signalset},x_1,\dots,x_N} &\; \sum_{i=1}^{N}\| y_i- A_{g_i}x_i \|^2 \\
    \text{s.t.}  \quad & \hat{\signalset} \in \mathcal{C}, \; x_i \in \hat{\signalset}
\end{align}
For example, low-rank matrix completion algorithms~\citep{cai2010singular,candes2009exact} assume that the signal set (columns of the matrix) consists of a $k$-dimensional linear subspace. Sparse subspace clustering with missing entries methods~\citep{yang2015sparse,ongie2017algebraic} assume that the signal set is a union of low-dimensional subspaces. These approaches are closely related to $k$-sparse dictionary learning approaches~\citep{studer2012dictionary}, which can also be trained  using only incomplete measurements. Union of subspace models are also closely related to a mixture of Gaussians with low-rank covariance matrices, which can also be used to learn from incomplete measurements~\citep{yang2015mixture}.

In some inverse problems, a direct inversion $y\mapsto x$ might be hard to obtain due to a high level of noise affecting the measurements. In this setting, it might be possible to identify the signal distribution  from incomplete data, without recovering each individual signal.  
For example, the signal distribution $p(x)$ can be modelled with a generative deep network. This strategy can be applied in learning problems with independent operators, $\group$-structured operators or $\group$-invariant signal sets with a single operator. 

\red{The case of multiple operators is summarized as follows. Let $q^{\textrm{data}}_g(y)$ be the empirical measurement distribution of observations linked to the $g$th operator. The learning strategy can be written as
\begin{align}
    \argmin_{\hat{p}(x)} &\; \sum_{g=1}^{\ntransf} d( \hat{q}_g(y) , q^{\textrm{data}}_g(y) ) \\
    \text{s.t.}  \quad & \hat{q}_g(y) = \int \ell(y|A_gx)\hat{p}(x)dx \quad g=1,\dots,\ntransf 
\end{align}
where $\ell(y|A_gx)$ denotes the distribution of measurements given $A_gx$ (i.e., the likelihood), and $d(\cdot,\cdot)$ is a (pseudo) divergence, which is also generally learned using generative adversarial networks (GANs). This approach has been adopted by AmbientGAN~\citep{bora2018ambientgan} for various imaging problems.}

\red{The case of a single operator and a $\group$-invariant model follows a similar learning strategy:
\begin{align}
    \argmin_{\hat{p}(x)} &\;  d( \hat{q}(y) , q^{\textrm{data}}(y) ) \\
    \text{s.t.}  \quad & \hat{q}(y) = \int \ell(y|Ax)\hat{p}(x)dx  \\
    & \hat{p}(x) \; \text{is $\group$-invariant}
\end{align}
where $q^{\textrm{data}}$ denotes the empirical distribution of the observations. This approach has been adopted by Cryo-GAN~\citep{gupta2020multi} for the cryo electron-microscopy problem.}


\subsection{Learning to Invert}  \label{subsec: learning to invert}
If we observe signals through multiple forward operators, an alternative  approach consists of directly learning the reconstruction function $f: (g,y)\mapsto x$ without explicitly learning the signal model. To this end, we can use the following training loss:
\begin{align} \label{eq:multop loss}
    \argmin_{f\in \mathcal{F}} &\; \sum_{i=1}^{N} \| y_i- A_{g_i} f(g_i,y_i) \|^2 .
\end{align}
The family of functions $\mathcal{F}$ is some parameterized function space \red{which captures the low-dimensionality of the signal set $\signalset$ (e.g., a deep neural network with an autoencoder architecture). Without the low-dimensional constraint on $\mathcal{F}$,  we can have $f(g,y)=A_g^{\dagger}y$  which is a minimizer of \cref{eq:multop loss}  but fails to learn the signal set.
For example, we can choose $f(g,y) = \phi(A_g^{\dagger}y)$ where $\phi(\cdot)$ plays the role of a denoiser function (independent of the forward operator) which projects $A_g^{\dagger}y$ to a low-dimensional set $\hat{\signalset}$. }
More generally, the network architecture can be an unrolled optimization algorithm~\citep{monga2021algorithm} which incorporates the forward operator $A_g$  and a denoiser subnetwork which is operator independent, e.g., acting as a \red{(low-dimensional)} proximal operator in the unrolled algorithm.

In cases where the model is assumed to be $\group$-invariant \JT{and $A$ is not itself equivariant (c.f.,~\Cref{theo: badAs})}, the reconstruction function $f:y\mapsto x$ can be learned by enforcing that the composition of the forward operator and reconstruction $f\circ A$ yields a $\group$-equivariant mapping on $\signalset$. To this end, we can enforce approximate equivariance via the following training loss:
\begin{align} \label{eq:equiv imaging}
    \argmin_{f\in \mathcal{F}} &\; \sum_{i=1}^{N} \| y_i- Af(y_i) \|^2 \\
    \text{s.t.}  &\quad f(AT_{g_i} \tilde{x}_i)  = T_{g_i} \tilde{x}_i \quad \forall \tilde{x}_i=f(y_i) \label{eq: equivariance constraint}
\end{align}
\JT{The additional equivariance constraint in~\cref{eq: equivariance constraint} stops the network from learning the identity mapping, since $I\circ A =A$ is not equivariant. }
This strategy was proposed in the equivariant imaging  framework~\citep{chen2021equivariant} and a similar approach is presented in~\citep{tachella2022sampling} for the case of multiple forward operators. In~\citep{chen2021equivariant}, the authors also consider adding an adversarial network to try to impose similarity between the estimated sample signal model $\hat{\signalset}$ and its transforms $\{T_g \hat{\signalset}\}_{g\in\group}$. The equivariant imaging loss in~\Cref{eq:equiv imaging} can be extended to handle noisy measurements~\citep{chen2021robust} using Stein's unbiased risk estimator~\citep{stein1981estimation}.


\section{Experiments} \label{sec:experiments}
We perform a series of numerical experiments  to illustrate the theoretical bounds presented in \Cref{sec:multiple ops,sec: group learning}. 

\subsection{Low-Dimensional Subspace Model}

We consider the problem of learning a $k$-dimensional subspace model from incomplete observations, where the signals $x_i$ are generated using a standard Gaussian distribution on the subspace. 
We first evaluate the case where each observation
$y_i$ is obtained by randomly choosing one out of $\ntransf$ operators $A_1,\dots,A_{\ntransf}\in\R{m \times n}$, each composed of iid Gaussian entries of mean 0 and standard deviation $n^{-1/2}$. In order to recover the signal matrix $X = [x_1,\dots,x_N]$, we solve the following matrix completion problem 
\begin{align}\label{eq:matrix comp}
    \arg\min_{X} \;&\|X\|_{*} \\
    \text{s.t. } A_{g_i}x_i &= y_i \quad \forall i=1,\dots,N \nonumber
\end{align}
where $\|\cdot\|_{*}$ denotes the nuclear norm. A recovery is considered successful if $\frac{\sum_i\|\hat{x}_i-x_i\|^{2}}{\sum_i\|x_i\|^{2}}<10^{-1}$, where $\hat{x}_i$ is the estimated signal for the $i$th sample. We use a standard matrix completion algorithm~\citep{cai2010singular} to solve~\Cref{eq:matrix comp}. The ambient dimension is fixed at $n=50$, and the experiment is repeated for $k=1,10,40$. For each experiment we set $N=150k$ in order to have enough samples to estimate the subspaces. \Cref{fig:multG_results} shows the probability of recovery over $25$ Monte Carlo trials for different numbers of measurements $m$ and operators $\ntransf$. The reconstruction probability exhibits a sharp transition which follows the generic sufficient condition $m>k+n/\ntransf$ of \Cref{theo: multiple op}. 

 \begin{figure}[t]
\centering
\includegraphics[width=\textwidth]{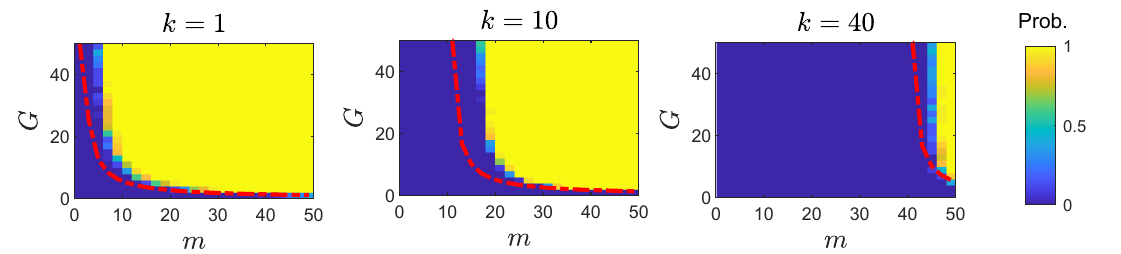}
\caption{Reconstruction probability of a $k$-dimensional subspace using incomplete measurements arising from $\ntransf$ independent operators for different $k$. The curve in red shows the bound of \Cref{theo: multiple op}, $m>k+n/\ntransf$.}
\label{fig:multG_results}
\end{figure}

Second, we perform the same experiment using $\group$-related operators $A_g = AT_g$.  \Cref{fig:results shifts} shows the probability of correct recovery for the group of shifts, where $\max_j  c_j/s_j=1$, whereas \Cref{fig:results ref} shows the recovery probability for the reflection group where $\max_j  c_j/s_j=n/2$. In both cases, the transition in probability of recovery follows $m \approx k + \max c_j/s_j$, i.e., requiring approximately $k$ less measurements than the sufficient condition of~\Cref{theo: one op}. As discussed in~\Cref{subsec: lowdim G models}, it is likely that these additional $k$ measurements are only required in worst-case scenarios such as~\Cref{ex: worst case ref}.

\begin{figure}[t]
     \centering
 \begin{subfigure}{0.4\textwidth}
     \centering
     \includegraphics[width=\textwidth]{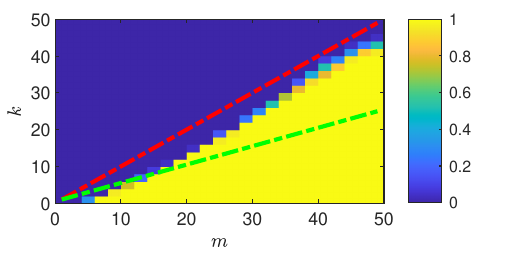}
     \caption{Shift}
     \label{fig:results shifts}
 \end{subfigure}
 \begin{subfigure}{0.4\textwidth}
     \centering
     \includegraphics[width=\textwidth]{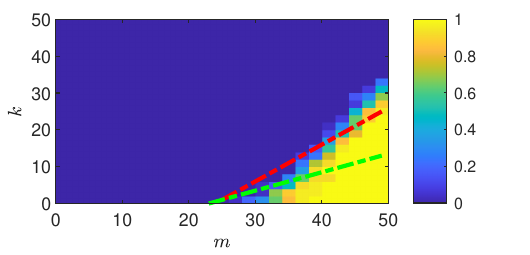}
     \caption{Reflection}
     \label{fig:results ref}
 \end{subfigure}
\caption{Reconstruction probability of a $k$-dimensional subspace using incomplete measurements arising from (a) shifted and (b) reflected operators $AT_g$. The bounds $m>k+\max_jc_j/s_j + 1$ and $m>2k+\max_jc_j/s_j + 1$ are plotted in the red and green dashed lines respectively.}
\end{figure}

\subsection{Deep Networks}

\subsubsection{MNIST via Multiple Operators}

We next consider the problem of directly learning the reconstruction function $f:(g,y)\mapsto x$ using measurements from  multiple compressed sensing operators, as described in~\Cref{subsec: learning to invert}. We train a network that aims to achieve data consistency for all the training data  via the unsupervised loss in~\Cref{eq:multop loss}.
We use the standard MNIST dataset which has an approximate box-counting dimension $k=12$~\citep{hein2005intrinsic}. The dataset contains $N=60000$ training samples, and these are partitioned such that $N/\ntransf$ different samples are observed via each operator. The entries of the forward operators are sampled from a Gaussian distribution with zero mean and variance $n^{-1}$. The test set consists of $10000$ samples, which are also randomly divided into $\ntransf$ parts, one per operator. We define $f(g,y) = \phi\circ A_g^{\dagger}$ where $\phi:\mathbb{R}^{n}\mapsto\mathbb{R}^{n}$ is a trainable network whose aim is to map $A_{g}^{\dagger}y$ to the signal set~$\signalset$. The networks are trained using the Adam optimizer.

 \begin{figure}[t]
\centering
\includegraphics[width=.6\columnwidth]{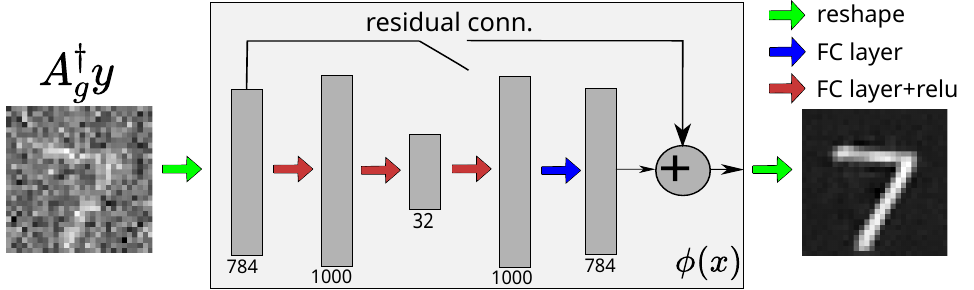}
\caption{Schematic of the fully connected autoencoder used in the MNIST experiments. Number of nodes are indicated at the bottom of each layer.}
\label{fig:network}
\end{figure}

\begin{table}[]
\centering
\begin{tabular}{|l|c|c|c|}
\hline
 & $\phi=$Id & $\phi$ w. residual & $\phi$ w/o residual \\ \hline
PSNR & 10.02 & 10.32 & 21.36  \\ \hline
\end{tabular}
\caption{Average PSNR in dB achieved by a residual network and a non-residual network for the MNIST reconstruction task with $\ntransf=40$ operators, each with $m=100$ measurements.}
\label{tab:residual}
\end{table}

When evaluating the ability of neural networks to perform unsupervised learning from incomplete data it is necessary to ensure that we are not just observing the inductive bias of the network, which has been shown can provide a powerful image model without any training~\citep{ulyanov2018deep,tachella2021nonlocal}. In order to minimize the impact of the inductive bias of the networks' architecture, we use fully connected layers which do not exploit any spatial image prior. It is easy to see that a valid solution to the optimization problem in~\cref{eq:multop loss} is just the identity $\phi(x) = x$.  This reflects the fact from the theory that in general the problem is not solvable and that we need to impose a low-dimensional prior on the signal. To avoid learning the identity, we require that $\phi(x)$ exploits the low-dimensionality of the signal set, so we use an autoencoder architecture with 3 hidden layers with 1000, 32 and 1000 neurons, as shown in~\Cref{fig:network}. We use relu non-linearities between layers, except at the output of the last layer.

\begin{figure}[t]
\centering
\begin{subfigure}{0.43\textwidth}
     \centering
\includegraphics[width=1\columnwidth]{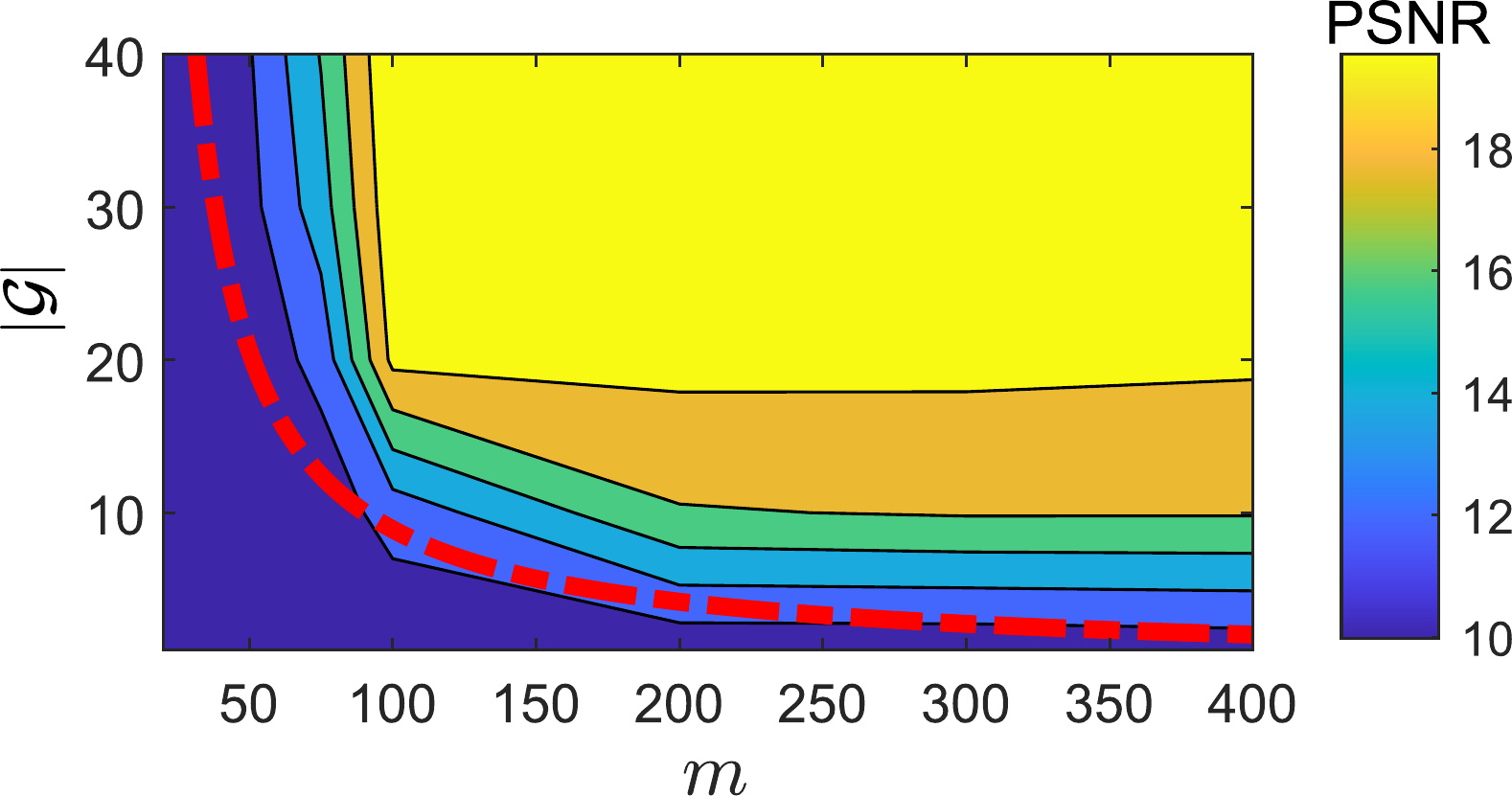}
\caption{}
\label{fig:mnist_cs_results}
\end{subfigure}
\hfill
\begin{subfigure}{0.49\textwidth}
\centering
\includegraphics[width=1\columnwidth]{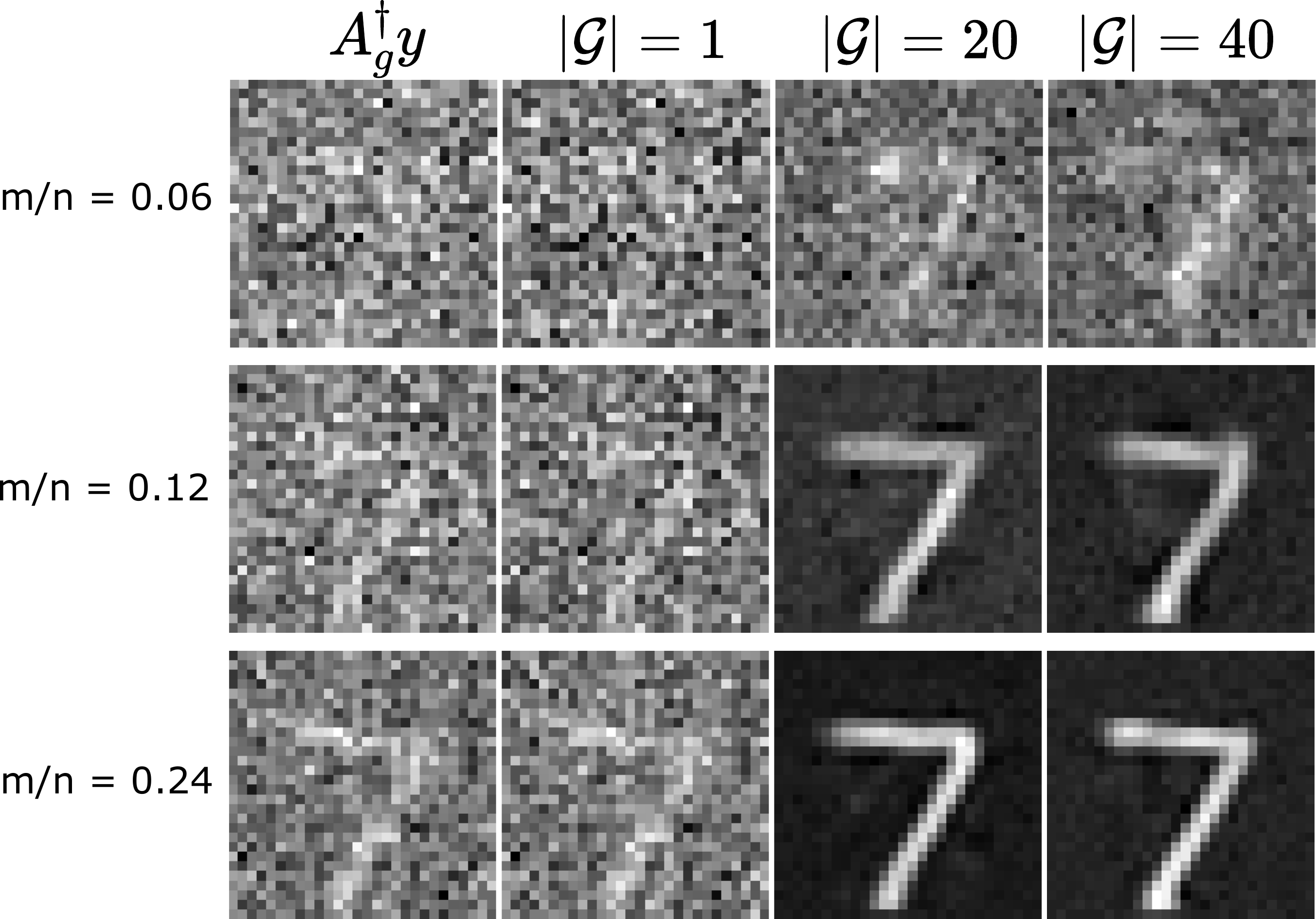}
\caption{}
\label{fig:results_cs}
\end{subfigure}
\caption{(a) Average test PSNR and (b) reconstructed images for the MNIST dataset, for different number of random Gaussian operators and measurements per operator. The curve in red shows the sufficient condition of \Cref{theo: multiple op}, $m>k+n/\ntransf$.}
\end{figure}

\begin{figure}[t]
\centering
\begin{subfigure}{0.45\textwidth}
     \centering
\includegraphics[width=1\columnwidth]{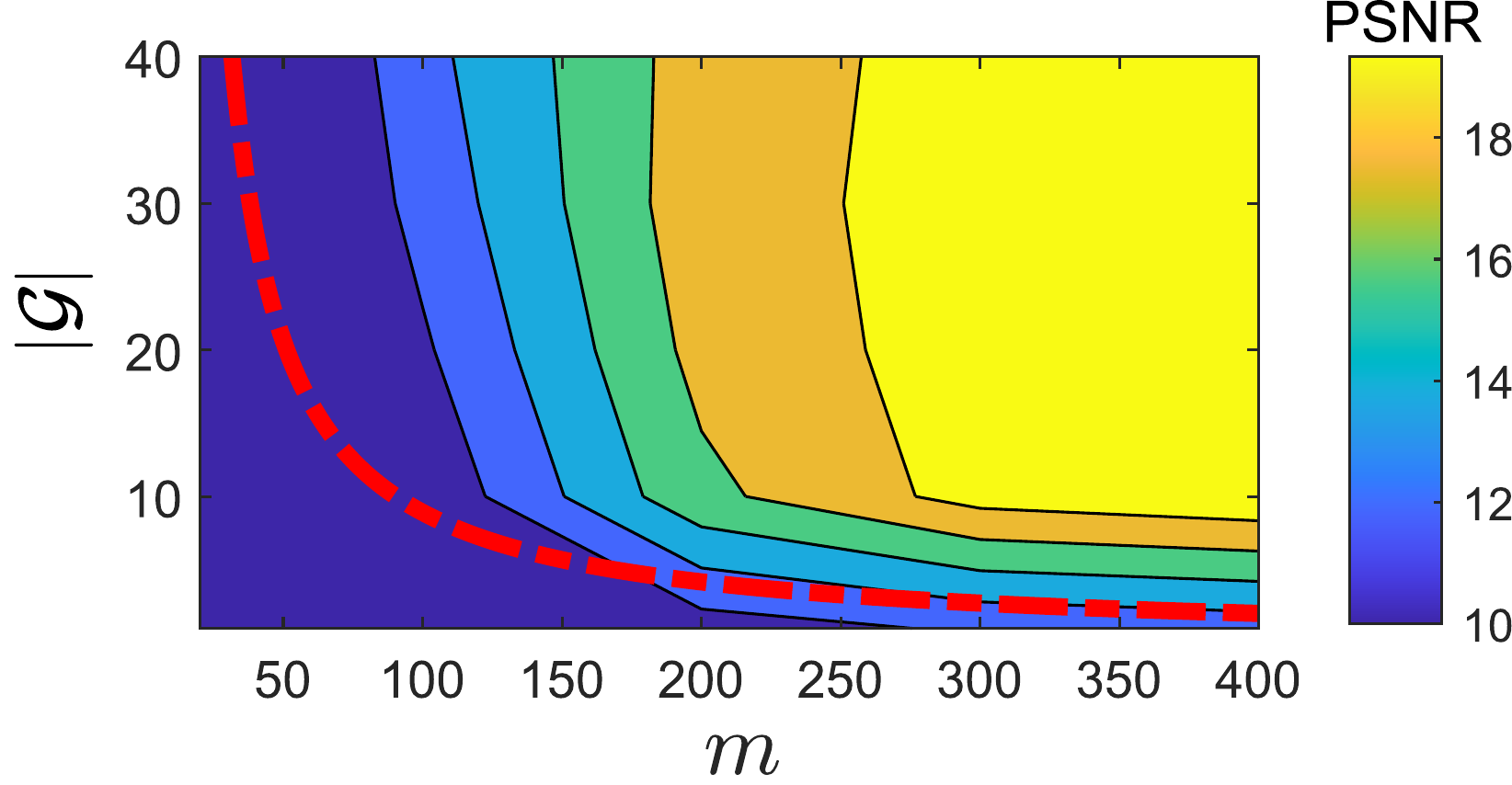}
\caption{}
\label{fig:mnist_ipt_results}
\end{subfigure}
\hfill
\begin{subfigure}{0.49\textwidth}
\centering
\includegraphics[width=1\columnwidth]{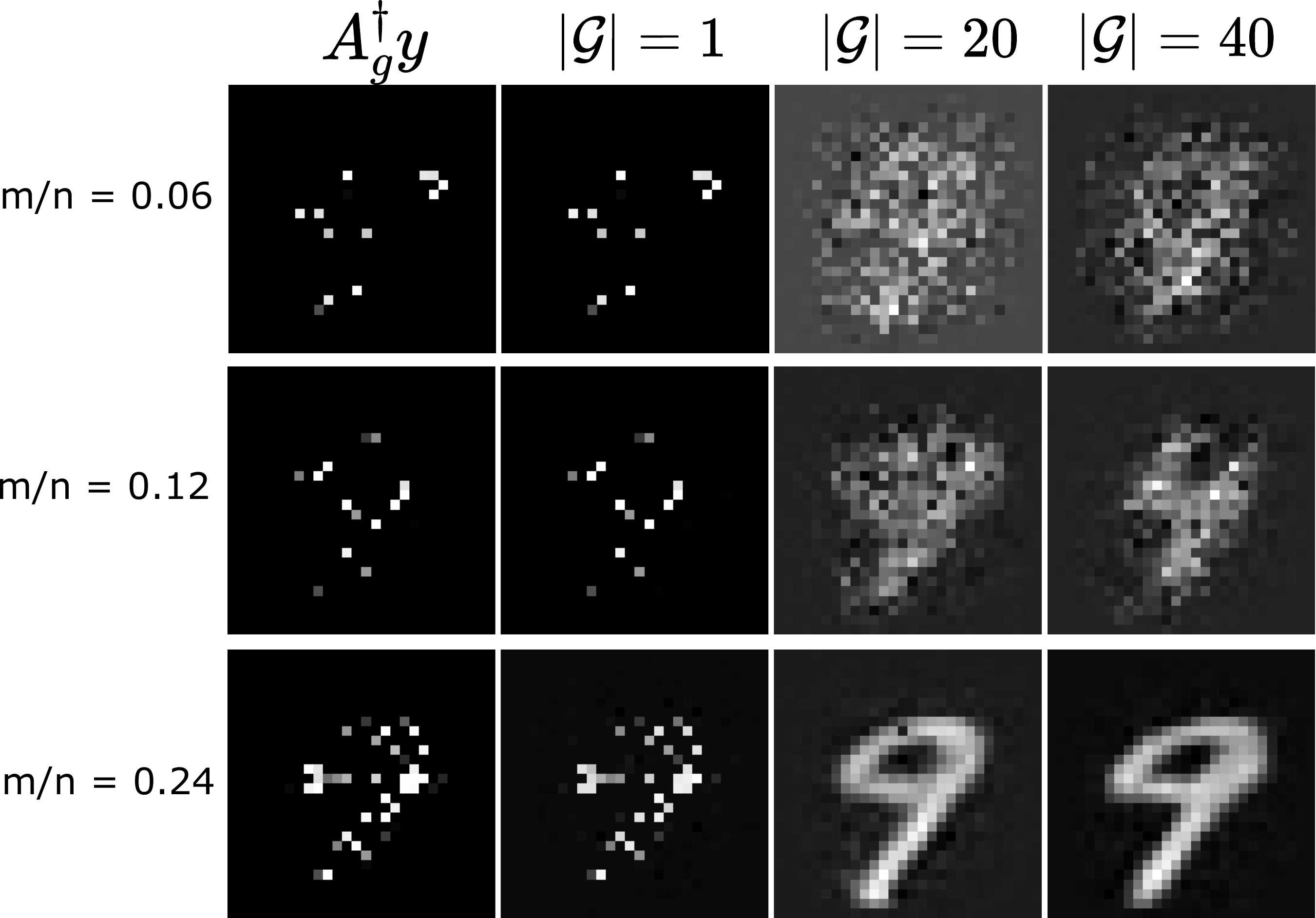}
\caption{}
\label{fig:results_ipt}
\end{subfigure}
\caption{(a) Average test PSNR and (b) reconstructed images for the MNIST dataset, for different number of random inpainting operators and measurements per operator. The curve in red shows the necessary condition of \Cref{theo: multiple op}, $m>k+n/\ntransf$.}
\end{figure}

\Cref{tab:residual} shows the performance for the case of $\ntransf=40$ operators with $m=100$ measurements each for $\phi(x)$ being (i) the identity, (ii) the autoencoder with a residual connection which allows it to learn the identity mapping and (iii) the autoencoder with no residual connection which enforces a low-dimensional representation. The network with residual connection fails to learn, obtaining a similar performance to the simple pseudo-inverse, whereas the autoencoder without residual obtains a significant improvement of more than $10$ dB. This is in line with our theory where we have seen that a requirement for unsupervised learning from incomplete measurements is that the model is low-dimensional. We next explore the ability of the autoencoder without residual connection to perform unsupervised learning as a function of the number of operators $\ntransf$ and measurements $m$.

\Cref{fig:mnist_cs_results} shows the average test peak-signal-to-noise ratio (PSNR) achieved by the trained model for $\ntransf=1,10,20,30,40$ and $m=1,100,200,300,400$. The results follow closely the bound presented in~\Cref{subsec: low-dim models} which is indicated by the red dashed line, as the network is  only able to  learn the reconstruction mapping when the sufficient condition $m>k+n/\ntransf$ is verified. In \JT{sensing} regimes below this condition, the performance is similar to simply applying the pseudo-inverse $A_g^{\dagger}$. \Cref{fig:results_cs} shows examples of reconstructed images for networks trained with different number of operators and measurements.

Second, we replace Gaussian operators for $\ntransf$ different random inpainting masks. The inpainting operators have a diagonal structure which has zero measure in $\mathbb{R}^{m\times n}$, however our sufficient condition still provides a reasonable lower bound on predicting the performance, as shown in~\Cref{fig:mnist_ipt_results}. It is likely that due to the coherence between measurement operators and images (both operators and MNIST images are sparse), more measurements are required to obtain good reconstructions than in the case of Gaussian operators. \Cref{fig:results_ipt} shows examples of reconstructed images for different number of operators and measurements.

\subsubsection{$\group$-Invariant MNIST}
Lastly, we consider the problem of learning a $\group$-invariant version of the MNIST dataset from incomplete measurements using deep convolutional neural networks. In this setting, the goal is to learn the reconstruction function $x_i=f(y_i)$ where $f$ is a trainable network. In order to learn from measurement data alone, we leverage the equivariant imaging approach of~\Cref{subsec: learning to invert}.
This technique has been very successfully applied to large scale CT and MRI imaging problems~\citep{chen2021equivariant,chen2021robust}. Here we demonstrate on a simple example that it also respects the measurement bounds presented in this paper. As in~\citep{chen2021equivariant,chen2021robust}, we compare 4 learning algorithms:
\begin{itemize}
    \item \textbf{Pseudo-inverse}: Reconstruction by applying the pseudo-inverse to the observed measurements $y_i=A^{\dagger}x_i$ (baseline of no learning).
    \item \textbf{Supervised}: Standard training using ground truth pairs $\{(x_i,y_i)\}_{i=1}^{N}$.
    \item \textbf{Unsupervised}: Training from incomplete measurements $\{y_i\}_{i=1}^{N}$ alone without enforcing $\group$-invariance of the reconstructed signal set.
    \item \textbf{Unsupervised equivariant}: Training from incomplete measurements $\{y_i\}_{i=1}^{N}$ alone. In order to enforce $\group$-invariance of the reconstructed signal set, we use the equivariant training loss of~\Cref{eq:equiv imaging}.  
\end{itemize}
In order to generate a $\group$-invariant dataset, we augment the standard MNIST dataset by applying the  transformations associated with the group action. In all cases, we use a single forward operator $A$ with iid Gaussian entries, and the U-Net network of~\citep{chen2021equivariant}. 
\begin{figure}[t]
     \centering
 \begin{subfigure}{0.42\textwidth}
     \centering
     \includegraphics[width=\textwidth]{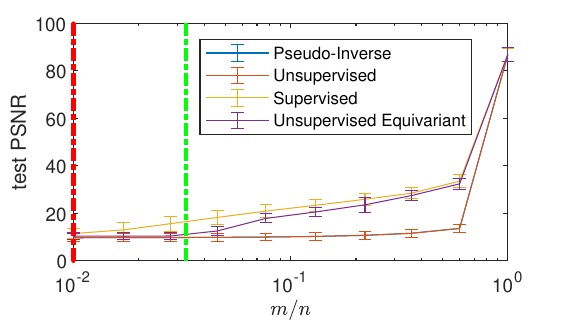}
     \caption{}
 \end{subfigure}
 \hfill
 \begin{subfigure}{0.49\textwidth}
     \centering
     \includegraphics[width=\textwidth]{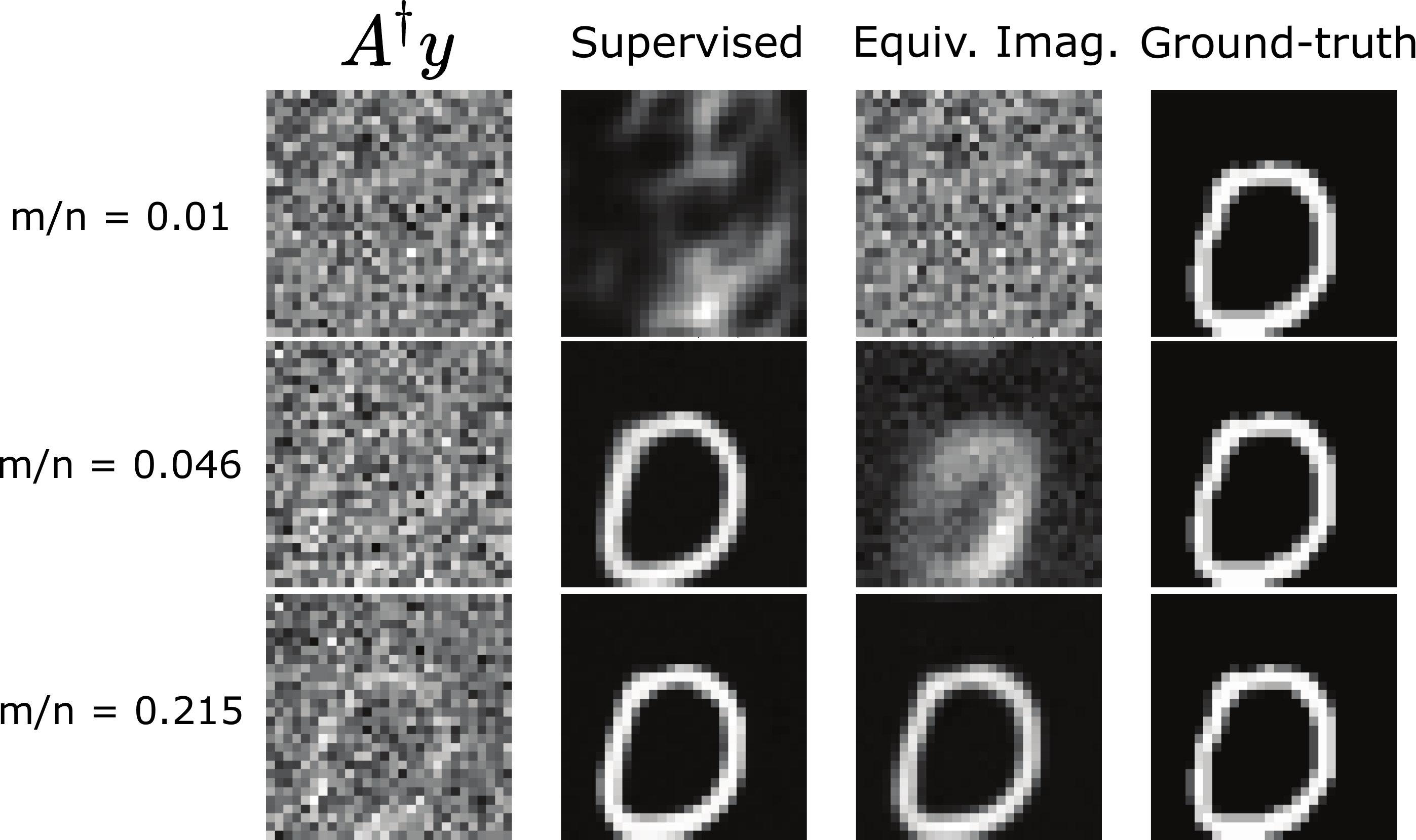}
     \caption{}
 \end{subfigure}
\caption{Shift-invariant MNIST experiments. (a) Average test PSNR  and (b) reconstructions obtained by the competing methods for different compression ratios. The unsupervised method with no equivariance and pseudo-inverse perform similarly.  \red{The green and red dash-dotted lines show the necessary and sufficient model identifiability conditions, respectively.}} 
\label{fig:mnist_shifts}
\end{figure}

First, we  evaluate the group of horizontal and vertical shifts. \red{This is an abelian group with irreducibles of dimension $s_j=1$, however it is not cyclic. As $\max_j  c_j/s_j=1$ for this group action, assuming Conjecture~\ref{conj: all groups} holds,} we expect that both supervised and unsupervised methods are able to reconstruct as long as $m>2k+1$. \Cref{fig:mnist_shifts} shows the reconstructed images and test peak-to-signal ratio (PSNR) for different compression ratios $m/n$. The unsupervised method with no equivariance fails to learn for all $m$, converging to the linear pseudo-inverse.
When the number of measurements is not enough for reconstruction, i.e., $m/n< 2k/n \approx 0.03$, both supervised and unsupervised methods fail. For $m/n>0.07$, the unsupervised equivariant method is able to perform as well as the fully supervised one as expected. There is an intermediate regime $0.03<m/n<0.07$ where the supervised method performs better than the unsupervised equivariant. While there are enough measurements for model uniqueness, this discrepancy might be attributed to not having enough measurements for stable model identification.

\begin{figure}[t]
     \centering
 \begin{subfigure}{0.45\textwidth}
     \centering
     \includegraphics[width=\textwidth]{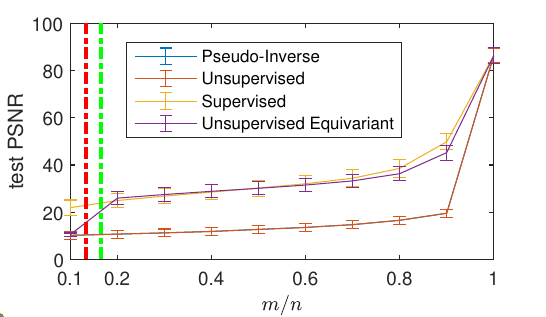}
     \caption{}
 \end{subfigure}
 \hfill
 \begin{subfigure}{0.49\textwidth}
     \centering
     \includegraphics[width=\textwidth]{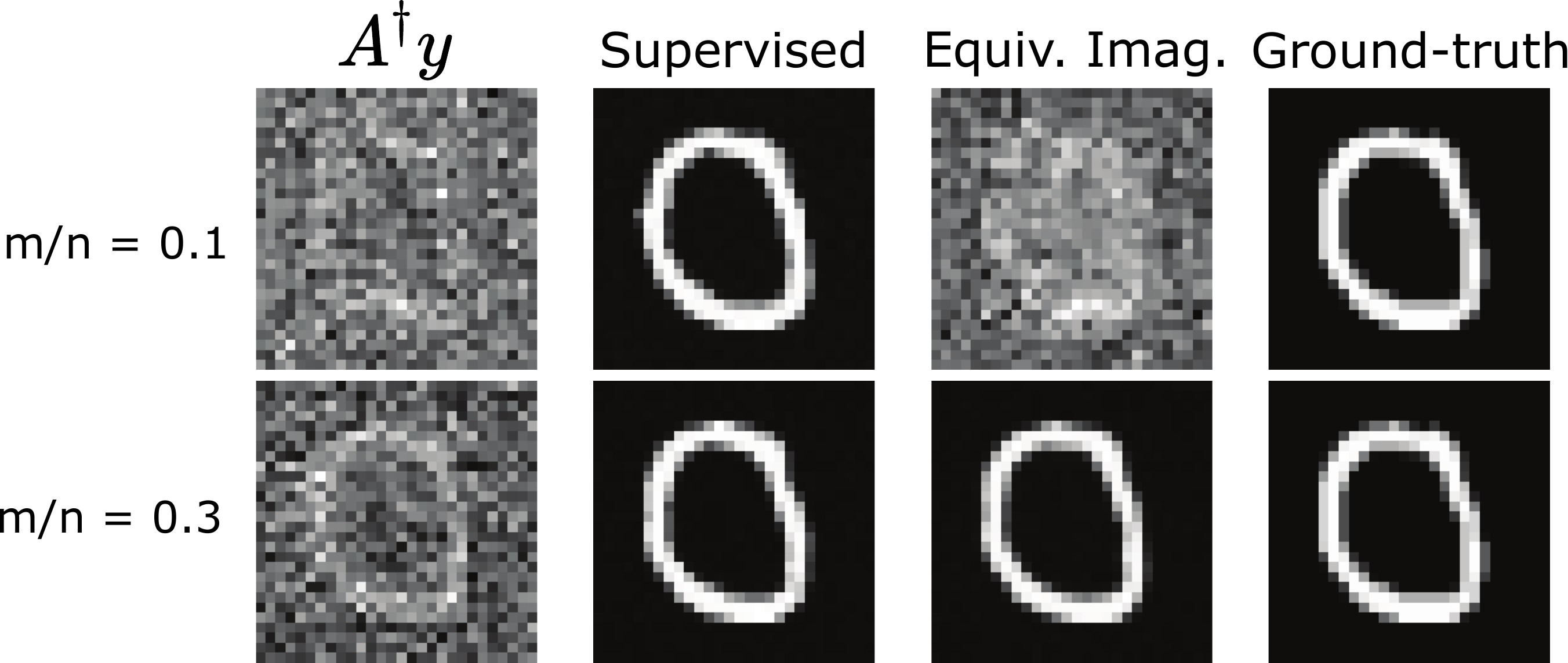}
     \caption{}
 \end{subfigure}
\caption{Reflection-invariant MNIST experiments.  (a) Average test PSNR  and (b) reconstructions obtained by the competing methods for different compression ratios. The unsupervised method with no equivariance and pseudo-inverse perform similarly. \red{The green and red dash-dotted lines show the necessary and sufficient model identifiability conditions, respectively.} }
\label{fig:mnist_refl}
\end{figure}

Secondly, we evaluate the group of 90 degree rotations and reflection, which has $\ntransf=8$ elements. This is a not a cyclic group (nor abelian) which is not covered by~\Cref{theo: one op}. However, we can obtain a bound from~\Cref{theo: one op} by considering the cyclic subgroup of 90 degree rotations. The full group action has $\max_j  c_j/s_j = 0.134 n$, whereas the largest cyclic subgroup  has $\max_j  c_j/s_j = 0.25 n$.  In both cases, the unsupervised equivariant method cannot achieve the same performance as the supervised one for small $m$. \Cref{fig:mnist_refl} shows the reconstructed images and test PSNR for different compression ratios $m/n$. As expected, the unsupervised method with no equivariance fails to learn for all $m$, converging to the linear pseudo-inverse.
For $m\geq 0.2 n$, both supervised and  unsupervised  equivariant perform equally well. This result is aligned with Conjecture~\ref{conj: all groups}, as the full group action offers a tighter sufficient condition of $m>0.16 n$, in comparison to the condition associated with the largest cyclic subgroup, $m>0.28n$. 

\section{Conclusions and Future Work}
We have presented fundamental bounds for learning models from incomplete measurements, either by using multiple measurement operators, or by exploiting a group symmetries of the signal set. Our bounds characterize the interplay between the fundamental properties of the inverse problem: the ambient dimension, the signal model dimension, and the number of measurement operators or symmetry of the model. Moreover, the bounds are agnostic of the learning algorithms and provide useful necessary and sufficient conditions for designing principled sensing strategies.

\JT{An interesting avenue of research is to extend the bounds in this paper to more general classes of signal models, including compressible  distributions~\citep{gribonval2012compressible}, e.g., Bernoulli-Gaussian or heavy tailed models, and other approximately low-dimensional models, e.g., models whose support is close to a low-dimensional set.}
We leave the study of robustness to noise as well as the extension of the theory for future work. 
Another interesting avenue of research is to extend our results to more general semigroups, whose elements do not necessarily have an inverse, but can be useful to capture certain symmetries in the signal model such as invariance to scale~\citep{florack1994linear, worrall2019deep}.

\section{Acknowledgements}
This work is supported by the ERC C-SENSE project (ERCADG-2015-694888).


\bibliographystyle{plainnat}
\bibliography{bibliography}

\begin{thebibliography}{53}
\providecommand{\natexlab}[1]{#1}
\providecommand{\url}[1]{\texttt{#1}}
\expandafter\ifx\csname urlstyle\endcsname\relax
  \providecommand{\doi}[1]{doi: #1}\else
  \providecommand{\doi}{doi: \begingroup \urlstyle{rm}\Url}\fi

\bibitem[Aghagolzadeh and Radha(2015)]{aghagolzadeh2015new}
Mohammad Aghagolzadeh and Hayder Radha.
\newblock New guarantees for blind compressed sensing.
\newblock In \emph{2015 53rd Annual Allerton Conference on Communication,
  Control, and Computing (Allerton)}, pages 1227--1234. IEEE, 2015.

\bibitem[Aguerrebere et~al.(2016)Aguerrebere, Delbracio, Bartesaghi, and
  Sapiro]{aguerrebere2016fundamental}
Cecilia Aguerrebere, Mauricio Delbracio, Alberto Bartesaghi, and Guillermo
  Sapiro.
\newblock Fundamental limits in multi-image alignment.
\newblock \emph{IEEE Transactions on Signal Processing}, 64\penalty0
  (21):\penalty0 5707--5722, 2016.

\bibitem[{Blumensath} and {Davies}(2009)]{blumensath2009uos}
T.~{Blumensath} and M.~E. {Davies}.
\newblock Sampling theorems for signals from the union of finite-dimensional
  linear subspaces.
\newblock \emph{IEEE Transactions on Information Theory}, 55\penalty0
  (4):\penalty0 1872--1882, 2009.
\newblock \doi{10.1109/TIT.2009.2013003}.

\bibitem[Bora et~al.(2018)Bora, Price, and Dimakis]{bora2018ambientgan}
Ashish Bora, Eric Price, and Alexandros~G. Dimakis.
\newblock Ambient{GAN}: Generative models from lossy measurements.
\newblock In \emph{International Conference on Learning Representations}, 2018.

\bibitem[Bourrier et~al.(2014)Bourrier, Davies, Peleg, P{\'e}rez, and
  Gribonval]{bourrier2014fundamental}
Anthony Bourrier, Mike~E Davies, Tomer Peleg, Patrick P{\'e}rez, and R{\'e}mi
  Gribonval.
\newblock Fundamental performance limits for ideal decoders in high-dimensional
  linear inverse problems.
\newblock \emph{IEEE Transactions on Information Theory}, 60\penalty0
  (12):\penalty0 7928--7946, 2014.

\bibitem[Bristow et~al.(2013)Bristow, Eriksson, and Lucey]{bristow2013fast}
Hilton Bristow, Anders Eriksson, and Simon Lucey.
\newblock Fast convolutional sparse coding.
\newblock In \emph{Proceedings of the IEEE Conference on Computer Vision and
  Pattern Recognition}, pages 391--398, 2013.

\bibitem[Cai et~al.(2010)Cai, Cand{\`e}s, and Shen]{cai2010singular}
Jian-Feng Cai, Emmanuel~J Cand{\`e}s, and Zuowei Shen.
\newblock A singular value thresholding algorithm for matrix completion.
\newblock \emph{SIAM Journal on optimization}, 20\penalty0 (4):\penalty0
  1956--1982, 2010.

\bibitem[Cand{\`e}s and Recht(2009)]{candes2009exact}
Emmanuel~J Cand{\`e}s and Benjamin Recht.
\newblock Exact matrix completion via convex optimization.
\newblock \emph{Foundations of Computational mathematics}, 9\penalty0
  (6):\penalty0 717, 2009.

\bibitem[Chen et~al.(2021)Chen, Tachella, and Davies]{chen2021equivariant}
Dongdong Chen, Juli\'an Tachella, and Mike~E. Davies.
\newblock Equivariant imaging: Learning beyond the range space.
\newblock In \emph{Proceedings of the IEEE/CVF International Conference on
  Computer Vision (ICCV)}, pages 4379--4388, October 2021.

\bibitem[Chen et~al.(2022)Chen, Tachella, and Davies]{chen2021robust}
Dongdong Chen, Juli{\'a}n Tachella, and Mike~E Davies.
\newblock Robust equivariant imaging: a fully unsupervised framework for
  learning to image from noisy and partial measurements.
\newblock In \emph{To Appear at Proceedings of the IEEE Conference on Computer
  Vision and Pattern Recognition (CVPR)}, 2022.

\bibitem[{Chen} et~al.(2015){Chen}, {Chi}, and {Goldsmith}]{chen2015rankone}
Y.~{Chen}, Y.~{Chi}, and A.~J. {Goldsmith}.
\newblock Exact and stable covariance estimation from quadratic sampling via
  convex programming.
\newblock \emph{IEEE Transactions on Information Theory}, 61\penalty0
  (7):\penalty0 4034--4059, 2015.
\newblock \doi{10.1109/TIT.2015.2429594}.

\bibitem[Cox et~al.(2013)Cox, Little, and OShea]{cox2013ideals}
David Cox, John Little, and Donal OShea.
\newblock \emph{Ideals, varieties, and algorithms: an introduction to
  computational algebraic geometry and commutative algebra}.
\newblock Springer Science \& Business Media, 2013.

\bibitem[Cram{\'e}r and Wold(1936)]{cramer1936some}
Harald Cram{\'e}r and Herman Wold.
\newblock Some theorems on distribution functions.
\newblock \emph{Journal of the London Mathematical Society}, 1\penalty0
  (4):\penalty0 290--294, 1936.

\bibitem[Donoho and Stark(1989)]{donoho1989uncertainty}
David~L Donoho and Philip~B Stark.
\newblock Uncertainty principles and signal recovery.
\newblock \emph{SIAM Journal on Applied Mathematics}, 49\penalty0 (3):\penalty0
  906--931, 1989.

\bibitem[Eaton(1989)]{eaton1989group}
Morris~L Eaton.
\newblock Group invariance applications in statistics.
\newblock In \emph{Regional conference series in Probability and Statistics},
  pages i--133. JSTOR, 1989.

\bibitem[Eriksson et~al.(2012)Eriksson, Balzano, and Nowak]{eriksson2012high}
Brian Eriksson, Laura Balzano, and Robert Nowak.
\newblock High-rank matrix completion.
\newblock In \emph{Artificial Intelligence and Statistics}, pages 373--381.
  PMLR, 2012.

\bibitem[Falconer(2004)]{falconer2004fractal}
Kenneth Falconer.
\newblock \emph{Fractal geometry: mathematical foundations and applications}.
\newblock John Wiley \& Sons, 2004.

\bibitem[Florack et~al.(1994)Florack, ter Haar~Romeny, Koenderink, and
  Viergever]{florack1994linear}
Luc~MJ Florack, Bart~M ter Haar~Romeny, Jan~J Koenderink, and Max~A Viergever.
\newblock Linear scale-space.
\newblock \emph{Journal of Mathematical Imaging and Vision}, 4\penalty0
  (4):\penalty0 325--351, 1994.

\bibitem[Gleichman and Eldar(2011)]{gleichman2011blind}
Sivan Gleichman and Yonina~C Eldar.
\newblock Blind compressed sensing.
\newblock \emph{IEEE Transactions on Information Theory}, 57\penalty0
  (10):\penalty0 6958--6975, 2011.

\bibitem[Gribonval et~al.(2012)Gribonval, Cevher, and
  Davies]{gribonval2012compressible}
R{\'e}mi Gribonval, Volkan Cevher, and Mike~E Davies.
\newblock Compressible distributions for high-dimensional statistics.
\newblock \emph{IEEE Transactions on Information Theory}, 58\penalty0
  (8):\penalty0 5016--5034, 2012.

\bibitem[Gribonval et~al.(2017)Gribonval, Blanchard, Keriven, and
  Traonmilin]{gribonval2017compressive}
R{\'e}mi Gribonval, Gilles Blanchard, Nicolas Keriven, and Yann Traonmilin.
\newblock Compressive statistical learning with random feature moments.
\newblock \emph{arXiv preprint arXiv:1706.07180}, 2017.

\bibitem[Guo and Davies(2015)]{guo2015near}
Chunli Guo and Mike~E Davies.
\newblock Near optimal compressed sensing without priors: Parametric sure
  approximate message passing.
\newblock \emph{IEEE Transactions on Signal Processing}, 63\penalty0
  (8):\penalty0 2130--2141, 2015.

\bibitem[Gupta et~al.(2020)Gupta, Phan, Yoo, and Unser]{gupta2020multi}
Harshit Gupta, Thong~H Phan, Jaejun Yoo, and Michael Unser.
\newblock Multi-cryo{GAN}: Reconstruction of continuous conformations in
  cryo-{EM} using generative adversarial networks.
\newblock In \emph{European Conference on Computer Vision}, pages 429--444.
  Springer, 2020.

\bibitem[Hein and Audibert(2005)]{hein2005intrinsic}
Matthias Hein and Jean-Yves Audibert.
\newblock Intrinsic dimensionality estimation of submanifolds in
  $\mathbb{R}^d$.
\newblock In \emph{Proceedings of the 22nd international conference on Machine
  learning (ICML)}, pages 289--296, 2005.

\bibitem[Jaritz et~al.(2018)Jaritz, Charette, Wirbel, Perrotton, and
  Nashashibi]{jaritz2018depth}
Maximilian Jaritz, Raoul~De Charette, Emilie Wirbel, Xavier Perrotton, and
  Fawzi Nashashibi.
\newblock Sparse and dense data with {CNN}s: Depth completion and semantic
  segmentation.
\newblock In \emph{2018 International Conference on 3D Vision (3DV)}, pages
  52--60, 2018.
\newblock \doi{10.1109/3DV.2018.00017}.

\bibitem[Jin et~al.(2017)Jin, McCann, Froustey, and Unser]{jin2017deep}
Kyong~Hwan Jin, Michael~T McCann, Emmanuel Froustey, and Michael Unser.
\newblock Deep convolutional neural network for inverse problems in imaging.
\newblock \emph{IEEE Transactions on Image Processing}, 26\penalty0
  (9):\penalty0 4509--4522, 2017.

\bibitem[Lehtinen et~al.(2018)Lehtinen, Munkberg, Hasselgren, Laine, Karras,
  Aittala, Aila, et~al.]{lehtinen2018noise2noise}
Jaakko Lehtinen, Jacob Munkberg, Jon Hasselgren, Samuli Laine, Tero Karras,
  Miika Aittala, Timo Aila, et~al.
\newblock Noise2{N}oise.
\newblock In \emph{International Conference on Machine Learning (ICML)}. PMLR,
  2018.

\bibitem[Liu et~al.(2020)Liu, Sun, Eldeniz, Gan, An, and Kamilov]{liu2020rare}
Jiaming Liu, Yu~Sun, Cihat Eldeniz, Weijie Gan, Hongyu An, and Ulugbek~S
  Kamilov.
\newblock {RARE}: Image reconstruction using deep priors learned without
  groundtruth.
\newblock \emph{IEEE Journal of Selected Topics in Signal Processing},
  14\penalty0 (6):\penalty0 1088--1099, 2020.

\bibitem[Metzler et~al.(2018)Metzler, Mousavi, Heckel, and
  Baraniuk]{metzler2018unsupervised}
Christopher~A Metzler, Ali Mousavi, Reinhard Heckel, and Richard~G Baraniuk.
\newblock Unsupervised learning with stein's unbiased risk estimator.
\newblock \emph{arXiv preprint arXiv:1805.10531}, 2018.

\bibitem[Monga et~al.(2021)Monga, Li, and Eldar]{monga2021algorithm}
Vishal Monga, Yuelong Li, and Yonina~C Eldar.
\newblock Algorithm unrolling: Interpretable, efficient deep learning for
  signal and image processing.
\newblock \emph{IEEE Signal Processing Magazine}, 38\penalty0 (2):\penalty0
  18--44, 2021.

\bibitem[Ongie et~al.(2017)Ongie, Willett, Nowak, and
  Balzano]{ongie2017algebraic}
Greg Ongie, Rebecca Willett, Robert~D Nowak, and Laura Balzano.
\newblock Algebraic variety models for high-rank matrix completion.
\newblock In \emph{International Conference on Machine Learning}, pages
  2691--2700. PMLR, 2017.

\bibitem[O’Toole et~al.(2018)O’Toole, Lindell, and
  Wetzstein]{o2018confocal}
Matthew O’Toole, David~B Lindell, and Gordon Wetzstein.
\newblock Confocal non-line-of-sight imaging based on the light-cone transform.
\newblock \emph{Nature}, 555\penalty0 (7696):\penalty0 338--341, 2018.

\bibitem[Pimentel-Alarcon and Nowak(2016)]{pimentel2016information}
Daniel Pimentel-Alarcon and Robert Nowak.
\newblock The information-theoretic requirements of subspace clustering with
  missing data.
\newblock In \emph{International Conference on Machine Learning}, pages
  802--810. PMLR, 2016.

\bibitem[Pimentel-Alarc{\'o}n et~al.(2016)Pimentel-Alarc{\'o}n, Boston, and
  Nowak]{pimentel2016characterization}
Daniel~L Pimentel-Alarc{\'o}n, Nigel Boston, and Robert~D Nowak.
\newblock A characterization of deterministic sampling patterns for low-rank
  matrix completion.
\newblock \emph{IEEE Journal of Selected Topics in Signal Processing},
  10\penalty0 (4):\penalty0 623--636, 2016.

\bibitem[Puy et~al.(2017)Puy, Davies, and Gribonval]{puy2017recipes}
Gilles Puy, Mike~E Davies, and R{\'e}mi Gribonval.
\newblock Recipes for stable linear embeddings from hilbert spaces to
  $\mathbb{R}^{m}$.
\newblock \emph{IEEE Transactions on Information Theory}, 63\penalty0
  (4):\penalty0 2171--2187, 2017.

\bibitem[Rapp et~al.(2020)Rapp, Tachella, Altmann, McLaughlin, and
  Goyal]{rapp2020advances}
Joshua Rapp, Julian Tachella, Yoann Altmann, Stephen McLaughlin, and Vivek~K
  Goyal.
\newblock Advances in single-photon lidar for autonomous vehicles: Working
  principles, challenges, and recent advances.
\newblock \emph{IEEE Signal Processing Magazine}, 37\penalty0 (4):\penalty0
  62--71, 2020.

\bibitem[Reeder(2014)]{reeder2014notes}
Mark Reeder.
\newblock \emph{Notes on representations of finite groups}.
\newblock 2014.

\bibitem[Robinson(2010)]{robinson2010dimensions}
James~C Robinson.
\newblock \emph{Dimensions, embeddings, and attractors}, volume 186.
\newblock Cambridge University Press, 2010.

\bibitem[Rudin et~al.(1992)Rudin, Osher, and Fatemi]{rudin1992nonlinear}
Leonid~I Rudin, Stanley Osher, and Emad Fatemi.
\newblock Nonlinear total variation based noise removal algorithms.
\newblock \emph{Physica D: nonlinear phenomena}, 60\penalty0 (1-4):\penalty0
  259--268, 1992.

\bibitem[Sauer et~al.(1991)Sauer, Yorke, and Casdagli]{sauer1991embedology}
Tim Sauer, James~A Yorke, and Martin Casdagli.
\newblock Embedology.
\newblock \emph{Journal of statistical Physics}, 65\penalty0 (3):\penalty0
  579--616, 1991.

\bibitem[Serre(1977)]{serre1977linear}
Jean-Pierre Serre.
\newblock \emph{Linear representations of finite groups}, volume~42.
\newblock Springer, 1977.

\bibitem[Silva et~al.(2011)Silva, Chen, Eldar, Sapiro, and
  Carin]{silva2011blind}
Jorge Silva, Minhua Chen, Yonina~C Eldar, Guillermo Sapiro, and Lawrence Carin.
\newblock Blind compressed sensing over a structured union of subspaces.
\newblock \emph{arXiv preprint arXiv:1103.2469}, 2011.

\bibitem[Stein(1981)]{stein1981estimation}
Charles~M Stein.
\newblock Estimation of the mean of a multivariate normal distribution.
\newblock \emph{The annals of Statistics}, pages 1135--1151, 1981.

\bibitem[Stiefel and F{\"a}ssler(2012)]{stiefel2012group}
E~Stiefel and A~F{\"a}ssler.
\newblock \emph{Group theoretical methods and their applications}.
\newblock Springer Science \& Business Media, 2012.

\bibitem[Studer and Baraniuk(2012)]{studer2012dictionary}
Christoph Studer and Richard~G Baraniuk.
\newblock Dictionary learning from sparsely corrupted or compressed signals.
\newblock In \emph{2012 IEEE International Conference on Acoustics, Speech and
  Signal Processing (ICASSP)}, pages 3341--3344. IEEE, 2012.

\bibitem[Sturmfels(2008)]{sturmfels2008algorithms}
Bernd Sturmfels.
\newblock \emph{Algorithms in invariant theory}.
\newblock Springer Science \& Business Media, 2008.

\bibitem[Tachella et~al.(2021)Tachella, Tang, and Davies]{tachella2021nonlocal}
Julian Tachella, Junqi Tang, and Mike Davies.
\newblock The neural tangent link between {CNN} denoisers and non-local
  filters.
\newblock In \emph{Proceedings of the IEEE/CVF Conference on Computer Vision
  and Pattern Recognition (CVPR)}, pages 8618--8627, June 2021.

\bibitem[Tachella et~al.(2022)Tachella, Chen, and Davies]{tachella2022sampling}
Juli{\'a}n Tachella, Dongdong Chen, and Mike Davies.
\newblock Unsupervised learning from incomplete measurements for inverse
  problems.
\newblock \emph{To Appear in NeurIPS}, 2022.

\bibitem[Tancik et~al.(2021)Tancik, Mildenhall, Wang, Schmidt, Srinivasan,
  Barron, and Ng]{tancik2021learned}
Matthew Tancik, Ben Mildenhall, Terrance Wang, Divi Schmidt, Pratul~P
  Srinivasan, Jonathan~T Barron, and Ren Ng.
\newblock Learned initializations for optimizing coordinate-based neural
  representations.
\newblock In \emph{Proceedings of the IEEE/CVF Conference on Computer Vision
  and Pattern Recognition}, pages 2846--2855, 2021.

\bibitem[Ulyanov et~al.(2018)Ulyanov, Vedaldi, and Lempitsky]{ulyanov2018deep}
Dmitry Ulyanov, Andrea Vedaldi, and Victor Lempitsky.
\newblock Deep image prior.
\newblock In \emph{Proceedings of the IEEE Conference on Computer Vision and
  Pattern Recognition (CVPR)}, pages 9446--9454, 2018.

\bibitem[Worrall and Welling(2019)]{worrall2019deep}
Daniel~E Worrall and Max Welling.
\newblock Deep scale-spaces: Equivariance over scale.
\newblock \emph{arXiv preprint arXiv:1905.11697}, 2019.

\bibitem[Yang et~al.(2015{\natexlab{a}})Yang, Robinson, and
  Vidal]{yang2015sparse}
Congyuan Yang, Daniel Robinson, and Rene Vidal.
\newblock Sparse subspace clustering with missing entries.
\newblock In \emph{International Conference on Machine Learning}, pages
  2463--2472. PMLR, 2015{\natexlab{a}}.

\bibitem[Yang et~al.(2015{\natexlab{b}})Yang, Liao, Yuan, Llull, Brady, Sapiro,
  and Carin]{yang2015mixture}
Jianbo Yang, Xuejun Liao, Xin Yuan, Patrick Llull, David~J. Brady, Guillermo
  Sapiro, and Lawrence Carin.
\newblock Compressive sensing by learning a {G}aussian mixture model from
  measurements.
\newblock \emph{IEEE Transactions on Image Processing}, 24\penalty0
  (1):\penalty0 106--119, 2015{\natexlab{b}}.
\newblock \doi{10.1109/TIP.2014.2365720}.

\end{thebibliography}

\appendix
\section*{Appendix A.}
\label{app:theorem}

\section{Proofs Preliminaries} \label{app: proof prelim} \label{sec: main proof}
Let $M$ be a matrix in $\C{m\times n}$ and $x$ be a vector in $\C{n}$. In the proofs, we will often use the `vec trick'
\begin{equation} \label{eq:vectrick}
    M x = (x^{\top} \otimes I_m)  \vect{M}
\end{equation}
where $\otimes$ denotes the Kronecker product, $I_m$ is an $m \times m$ identity matrix and $\vect{M}\in \C{mn}$ is the column-wise vectorization of the matrix $M$. 
We will also use the following result by~\citep{sauer1991embedology}:

\begin{lemma} [Lemmas 4.5 and 4.6 in \citep{sauer1991embedology}] \label{lemma:sauer}

Let $S$ be a bounded subset of $\R{n}$, and let $G_0, G_1,\dots,G_t$ be Lipschitz maps from $S$ to $\R{m}$. For each integer $r\geq0$, let $S_r$ be the subset of $z\in S$ such that the rank of the $m\times t$ matrix 
\begin{equation}
    \Phi_z = [G_1(z),\dots,G_t(z)]
\end{equation}
is $r$, and let $\bdim{S_r} = k_r$. For each $\alpha\in\R{t}$ define $G_{\alpha}(z)=G_0+\sum_{i=1}^{t}\alpha_i G_i(z)=G_0 + \Phi_z\alpha $. 
If for all integers $r\geq 0$ we have that $r > k_r$, then $G_{\alpha}^{-1}(0)$ is empty for almost every $\alpha\in \R{t}$.
\end{lemma}

The proof of this result follows standard covering arguments and may be sketched as follows. From the dimensionality assumption, the set $S_r$ can be essentially covered by  $\mathcal{O}(\epsilon^{-k_r})$ $\epsilon$-balls. Furthermore, for any $z \in S_r$, the probability (measured w.r.t. $\alpha \in \R{t}$) that $G_{\alpha}(z)$ maps to the neighborhood of $0$ scales as $\epsilon^r$. Hence the probability of this happening for any of the points in the cover scales as $\epsilon^{r-k_r}$. If we take $r> k_r$ then the probability of such an event tends to zero as we shrink $\epsilon$. Full details can be found in the proofs in~\citep{sauer1991embedology}.

\section{Proof of~\Cref{theo: multiple op}} \label{app: multop proof}
\begin{proof}
In order to have model uniqueness, we require that the inferred signal set $\hat{\signalset}$ defined in \Cref{eq:inferred set} equals the true set $\signalset$, or equivalently that their difference 
\begin{equation} \label{eq:diff inf multop}
    \hat{\signalset}\setminus{\signalset} = \{ v\in \R{n} \setminus{\signalset} | \;  A_1 (x_1 - v) = \dots =  
    A_{\ntransf} (x_{\ntransf} -v) = 0, \; x_1,\dots,x_{\ntransf}\in \signalset \}
\end{equation}
is empty, where $\setminus{}$ denotes set difference. Let $S\subset \R{n(\ntransf+1)}$ be the set of all vectors $z=[v,x_1,\dots,x_{\ntransf}]^{\top}$ with  $v\in \R{n}\setminus{\signalset}$ and $x_1,\dots,x_{\ntransf}\in \signalset$. The difference set defined in \Cref{eq:diff inf multop} is empty if and only if for any $z\in S$  we have
\begin{align}  \label{eq:multA}
  \underbrace{\begin{bmatrix}
  - A_1 & A_{1} &  & \\ 
   \vdots & & \ddots & \\ 
  - A_{\ntransf}& &  & A_{\ntransf}
 \end{bmatrix}}_{G_\alpha\in\R{m\ntransf \times n(\ntransf+1) }}
  \underbrace{\begin{bmatrix}
  v \\
   x_{1}\\
   \vdots \\ 
    x_{\ntransf}\\
 \end{bmatrix}}_{z\in S \subset \R{n(\ntransf+1)}} &\neq 0 \\
 G_{\alpha}(z) &\neq 0
\end{align}
where $G_\alpha$ maps $z\in S\subset \R{n(\ntransf+1)}$ to  $\R{m\ntransf}$. Let $\alpha = [\vect{A_1}^{\top},\dots,\vect{A_{\ntransf}}^{\top}]^{\top}\in \R{mn\ntransf}$, then as a function of $\alpha$ we can also write \Cref{eq:multA} as 
 \begin{equation} \label{eq: alpha mult}
  \begin{bmatrix}
   (x_1-v)^{\top} \otimes I_{m} &  & \\ 
   & \ddots & \\ 
   &  & (x_{\ntransf}-v)^{\top} \otimes I_{m}
 \end{bmatrix} \alpha \neq 0
\end{equation}
where we used the `vec trick' in~\cref{eq:vectrick}. As $v$ does not belong to the signal set, the matrix on the left hand side of \Cref{eq: alpha mult} has rank $m\ntransf$ for all  $z\in S$. We treat the cases of bounded and conic signal sets separately, showing in both cases that, for almost every $\alpha\in \R{mn\ntransf}$, the condition in \Cref{eq: alpha mult} holds for all $z\in S$ if $m>k+n/\ntransf$:


\begin{description}
    \item[Bounded signal set] Decompose $S$ into a countable union of bounded subsets $S=\bigcup_{\rho\geq1}S_\rho$ defined as
    \begin{equation}
      S_\rho =  \{ z \in \R{n(\ntransf+1)} \; | z = [v^{\top},x_1^{\top},\dots,x_{\ntransf}^{\top} ]^{\top}, x_1,\dots,x_{\ntransf}\in \signalset,\; \|v\|_2\leq \rho \}.
    \end{equation}
    where $\bdim{S_\rho}\leq k\ntransf+n$. Thus, Lemma~\ref{lemma:sauer} states that for almost every $\alpha$ all $z\in S$ verifies \Cref{eq: alpha mult} if $m>k+n/\ntransf$.
    
    \item[Conic signal set] If the signal set is conic, then $S$ is also conic. Let $B$ be a bounded set $B$ containing an open neighbourhood of $0$. As $\bdim{S\cap B}\leq \ntransf k+n$, Lemma~\ref{lemma:sauer} states that for almost every $\alpha$, all $z\in S$ verifies \Cref{eq: alpha mult} as long as $m> k + n\ntransf$. 
\end{description}
\end{proof}

\section{Proofs for $\group$-invariant models} \label{app: group proofs}
We next present a lemma which will be useful for the proofs of Proposition~\ref{prop: necessary G structure} and~\Cref{theo: one op}.

\begin{lemma} \label{lem: block-diagonal orbit}
Let  $v\in\C{n}$ and let the decomposition of  $v$ into the $J$ invariant subspaces of a $\group$-action can be written as:
\begin{equation}
    v = \begin{bmatrix}
     v_{1} \\
     \vdots \\
     v_{J}
    \end{bmatrix} \in \C{n}
    \text{ where }
    v_j = \begin{bmatrix}
     v_j^{1} \\
     \vdots \\
     v_j^{c_j}
    \end{bmatrix} \in \C{s_jc_j}
    .
\end{equation}
where $v_{j}^{\ell} \in \C{s_j}$ corresponds to the $\ell$th copy out of $c_j$ multiplicities of the $j$th invariant subspace. \red{We have
\begin{equation}
  \frac{1}{\ntransf}\sum_{g\in\group} T_g vv^{\top} T_g^{\top} = M_v M_v^{\top}
\end{equation}
with
\begin{equation} \label{eq: block_matrix}
 M_v =  \begin{bmatrix}
    [v_1^{1},v_1^{2},\dots,v_1^{c_1}]^{\top} \otimes I_{s_1}  & & \\
    & \ddots & \\
    & & [v_J^{1},v_J^{2},\dots,v_J^{c_J}]^{\top} \otimes I_{s_J} \\
    \end{bmatrix} .
\end{equation}
The result also holds for infinite compact groups by replacing the sum over the group with an integral.}
\end{lemma}

\begin{proof}
Without loss of generality, we assume that the linear representation of $\group$ is block-diagonalized in the canonical basis ($F$ in \Cref{theo: irreps} equals the identity). 
Using the `vec trick' in~\cref{eq:vectrick} and the decomposition of $v$ into invariant subspaces, \red{we have
\begin{align}
    T_gv &= 
     \begin{bmatrix}
    \rho_1(g)v_1^1
    \\ \vdots \\
    \rho_J(g)v_J^{c_J}
    \end{bmatrix}
    \\
 &=  \begin{bmatrix}
   \left( [v_1^{1},v_1^{2},\dots,v_1^{c_1}]^{\top} \otimes I_{s_1} \right)\text{vec}\rho_1(g) \\ \vdots \\
    \left( [v_J^{1},v_J^{2},\dots,v_J^{c_J}]^{\top} \otimes I_{s_J} \right)\text{vec}\rho_J(g)
    \end{bmatrix}
 \\ \label{eq:block decomp Tv}
 & = \begin{bmatrix}
  R_1\text{vec}\rho_1(g) \\ \vdots \\
   R_J\text{vec}\rho_J(g)
    \end{bmatrix}
\end{align}
where the last line defines $R_j := \left([v_j^{1},v_j^{2},\dots,v_1^{c_j}]^{\top} \otimes I_{s_j} \right) \in \C{s_jc_j\times s_j^2}$ for $j=1,\dots,J$. Using this decomposition, we can compute the  $(j,j')$ block of $\sum_{g\in\group} T_gvv^{\top}T_g^{\top}$ as 
\begin{align}
    \frac{1}{\ntransf}\sum_{g\in\group}  R_j \text{vec}\rho_{j}(g) \text{vec}\rho_{j'}(g)^{\top} R_{j'}^{\top} 
   &=  R_j  \left( \frac{1}{\ntransf}\sum_{g\in\group}\text{vec}\rho_{j}(g) \text{vec}\rho_{j'}(g)^{\top}  \right)R_{j'}^{\top}
\end{align}
where the middle term can be evaluated using the orthogonality relations of irreducible representations~\citep[Chapter~2]{serre1977linear}:
\begin{equation}
    \frac{1}{\ntransf} \sum_{g\in\group}\text{vec}\rho_{j}(g) \text{vec}\rho_{j'}(g)^{\top} = \begin{cases}
    I_{s_j^2} \quad \text{if} \quad j=j' \\
    0 \quad \text{ otherwise}
    \end{cases}.
\end{equation}
Considering all blocks, we have\footnote{For complex matrices $M$, $M^{\top}$ denotes the conjugate transpose.} $ \frac{1}{\ntransf}\sum_{g\in\group} T_g vv^{\top} T_g^{\top} = M_v M_v^{\top}$
with
\begin{equation}
 M_v =  \begin{bmatrix}
    R_1  & & \\
    & \ddots & \\
    & & R_J \\
    \end{bmatrix} .
\end{equation}}

\red{For infinite compact groups acting on $\C{n}$, we can obtain the same result by replacing the sum over the group by an integral~\citep[Chapter~4]{serre1977linear}, where the orthogonality relations are given by
\begin{equation}
    \int_{\group}\text{vec}\rho_{j}(g) \text{vec}\rho_{j'}(g)^{\top}dg = \begin{cases}
    I_{s_j^2} \quad \text{if} \quad j=j' \\
    0 \quad \text{ otherwise}
    \end{cases}.
\end{equation}}

\end{proof}

\textbf{Proof of Proposition~\ref{prop: necessary G structure}.}
\begin{proof}
\red{We begin with the case of finite groups, and then extend it to infinite (but compact) groups.} In order to have model uniqueness it is necessary that the matrix 
\begin{equation} \label{eq: ATs}
   M^{\top} = \begin{bmatrix}
     AT_1 \\ 
     \vdots \\
     AT_{\ntransf} 
    \end{bmatrix} 
\end{equation}
has rank $n$. This matrix contains the orbits of the measurement vectors $\{a_i\}_{i=1}^{m}$ (the rows of $A$). 
Letting $R= \frac{1}{\ntransf}MM^{\top}$ and using Lemma~\ref{lem: block-diagonal orbit}, we have  
\begin{align}\label{eq: tildeM}
 R &= \sum_{i=1}^{m}  \frac{1}{\ntransf}\sum_{g\in\group}  T_g a_i a_i^{\top}T_g^{\top}  \\
 &= \sum_{i=1}^{m}  M_{a_i} M_{a_i}^{\top} \\
 &=[M_{a_1}, \dots, M_{a_m} ] [M_{a_1}, \dots, M_{a_m} ]^{\top} \label{eq: last_tildeM}
\end{align}
where  $M_{a_i}\in \C{n \times \sum_j s_j^2}$ is the block-diagonal matrix in~\cref{eq: block_matrix} associated to $a_i$. Note that $\rk{M^{\top}} = \rk{R} = \rk{[M_{a_1}, \dots, M_{a_m} ]}$. Furthermore, due to the block-diagonal structure of the submatrices $M_{a_i}$, $[M_{a_1}, \dots, M_{a_m} ]$ can also be rearranged in block diagonal form with blocks of size $s_j c_j \times m s_j^2$. Thus, $M^{\top}$ has rank $n$ only if all the blocks verify $ms_j^2\geq s_jc_j$, which yields the bound
\begin{equation}
m\geq \max_j  c_j/s_j.
\end{equation} 
\red{If the group has infinite elements, the matrix in \cref{eq: ATs} is not well-defined as it would have infinite entries. However, as the decomposition into finite invariant subspaces and Lemma~\ref{lem: block-diagonal orbit} still hold, we can compute the dimension of the subspace spanned by $\{T_ga_i\}_{g\in\group, i=1,\dots,m}$ with \cref{eq: tildeM} by replacing the finite sum over group elements by an integral and obtain the bound $m\geq \max_j  c_j/s_j$.}

\end{proof}

We now present two useful technical lemmas  for proving~\Cref{theo: one op}. 

\begin{lemma} \label{lem: v-variety}
Let $T\in \R{n\times n}$ be the linear representation of the generator of a finite cyclic group of order $\ntransf$. Let $v \in \R{n}$ and  $B \in \R{n \times r+1}$ with $r+1\leq n$, such that $M = B-[Tv,\dots,T^{r+1}v]$ has rank $r$ and the first $r$ columns are linearly independent. The set $\Omega\subset \R{n}$ of $v$ verifying the rank assumption is contained in an affine variety of dimension at most $r+\sum_{j\in\mathcal{J}} c_j$ where $c_j$ denotes the multiplicity of the $j$th irreducible representation of the group action and $\mathcal{J}$ is a subset of $r$ out of $\ntransf$ irreducibles.
\end{lemma}

\begin{proof}
By assumption, $M = B-[Tv,\dots,T^{r+1}v]$ has rank $r$. All $(r+1)\times (r+1)$ minors of matrices with rank at most $r$ are necessarily zero. Minors are given by polynomial equations on the entries of $M$. Thus, the (at most) rank-$r$ condition on $M$ can be translated into a set of polynomial equations on $v$ which must equal zero.
As the zero set of polynomial equations, the subset of $v\in\R{n}$ which verify this constraint is a variety~\citep{cox2013ideals}.  We will use these polynomial equations to show that this set has dimension at most $r+\sum_{j\in\mathcal{J}} c_j$ where $\mathcal{J}$ is a subset of $r$ out of $\ntransf$ irreducibles. The polynomials depend on both $B$ and $v$, thus the $v$-variety will vary smoothly as a function of $B$. In order to simplify the analysis, we study the set of complex $v\in\C{n}$ that verify the rank condition, noting that the set of real $v\in \R{n}$ is just a  subset of the complex setting.

Using~\Cref{theo: irreps}, the matrix $M$ can be written as
\begin{align}
    M
    & = \begin{bmatrix}
   B_{1,1} -\rho_1(1)v_1 &  B_{1,2} -\rho_1(2)v_1 & \dots & B_{1,r+1} -\rho_1(r+1)v_1\\
B_{2,1} -\rho_{j_2}(1)v_2 & B_{2,2} -\rho_{j_2}(2)v_2  & \dots &  B_{2,r+1} -\rho_{j_2}(r+1)v_2\\
   \vdots &  \vdots & \ddots & \vdots\\
   B_{n,1} -\rho_J(1)v_n  &  B_{n,2} -\rho_J(2)v_n & \dots & B_{n,r+1} -\rho_J(r+1)v_n \\
    \end{bmatrix} 
\end{align}
where we use the index $j_i$ to indicate the irreducible representation associated with the $i$th row. For example, if $J=2$ and $c_1=c_2=2$, we have that $j_1 = j_2 = 1$ and $j_3 = j_4 = 2$. Moreover, as we are dealing with cyclic groups, we have that $\rho_j(g)=e^{- \mathrm{i} 2\pi jg/\ntransf}$~\citep[Chapter~5]{serre1977linear}.

By assumption, there is at least one $r\times r$ invertible submatrix within the first $r$ columns.  Without loss of generality, we assume that this submatrix is given by the first $r$ rows. For any other choice of rows, the results are identical with a re-indexing. We then consider the $n-r$  minors of size $(r+1)\times (r+1)$ corresponding to the first $r+1$ columns together with the first $r$ rows and the $i$th row:
\begin{equation}
   \text{det} \begin{bmatrix}
   B_{1,1} -\rho_1(1)v_1 &  \dots & B_{1,r+1} -\rho_1(r+1)v_1\\
   \vdots &   \ddots & \vdots\\
   B_{r,1} -\rho_{j_r}(1)v_r  & \dots & B_{r,r+1} -\rho_{j_r}(r+1)v_r \\
   B_{i,1} -\rho_{j_{i}}(1)v_{i} & \dots & B_{i,r+1} -\rho_{j_{i}}(r+1)v_{i} \\
    \end{bmatrix} 
\end{equation}
for $i=r+1,\dots,n$. 
Applying Laplace's expansion on the last row of each minor, we have that 
\begin{equation}
     \sum_{g=1}^{r+1} \left(B_{i,g} - \rho_{j_i}(g)v_{i} \right) \mu_g = 0
\end{equation}
where $\mu_1,\dots,\mu_{r+1}$ are the determinants of $r\times r$ matrices which only depend on the first $r+1$ columns of $B$ and $v_1,\dots,v_r$. Let $\mu = [\mu_1,\dots,\mu_{r+1},0,\dots,0]^{\top}\in\C{\ntransf}$, where $\mu_{g} = 0$ for all $g> r+1$. Note that $\mu$ cannot be identically zero, as $\mu_{r+1}\neq 0$ is the determinant of the submatrix with the first $r$ columns and rows, which is invertible. The equations can be rewritten as 
\begin{equation} \label{eq: mat constraint mu}
  \begin{bmatrix} 
  \hat{\mu}_{j_{r+1}}& &  \\
  & \ddots & \\
  & & \hat{\mu}_{j_n}
  \end{bmatrix} \begin{bmatrix}
  v_{r+1} \\
  \vdots \\
  v_n
  \end{bmatrix} = \begin{bmatrix}
  d_{r+1} \\
  \vdots \\
  d_n
  \end{bmatrix}
\end{equation}
where $\hat{\mu}_j = \sum_{g=1}^{\ntransf} \rho_{j}(g) \mu_g $ are the $j$th coefficient of the discrete Fourier transform of the vector $\mu$ (which has at most $r+1$ consecutive nonzero elements) and $d_i=\sum_{g=1}^{r+1} B_{i,g}\mu_g$. For all $i\in \{r+1,\dots,n\}$ where $\hat{\mu}_{j_i}\neq 0$, we have $v_i=d_i/\hat{\mu}_{j_i}$ in order to satisfy the rank-$r$ constraint. As the Fourier transform of a vector with at most $r+1$ non-zero elements which are consecutive has at most $r$ zero coefficients~\citep[Lemma~5]{donoho1989uncertainty}, we have that at most $r$ coefficients $\hat{\mu}_{j}=0$. As each $\hat{\mu}_j$ is repeated at most $c_j$ times in~\cref{eq: mat constraint mu}, there are at most $\sum_{j\in\mathcal{J}} c_j$ zeros along the diagonal in~\cref{eq: mat constraint mu}, where $\mathcal{J}$ is a subset of $r$ out of $\ntransf$ irreducibles. Locally, we are free to vary the first $r$ components $v_1,\dots,v_r$ without changing the rank of the submatrix. From \Cref{eq: mat constraint mu}, we are also free to locally vary at most $\sum_{j \in \mathcal{J}} c_j$ components of the remaining $v_i$ associated with the zero set of $\hat{\mu}_{j_i}$. Hence, the set of $v$ that verifies the rank-$r$ constraint has a dimension equal or smaller than $r + \sum_{j\in\mathcal{J}} c_j$.
\end{proof}

\begin{lemma} \label{lem: cyclic subgroup}
For any infinite compact cyclic group $\group_1$ acting on $\R{n}$, there is a finite cyclic subgroup $\group_2\subset\group_1$ such that the restriction of the linear representation of $\group_1$ to $\group_2$ has the same multiplicities of the irreducible representations.
\end{lemma}

\begin{proof}
Following \Cref{theo: irreps}, the representation of a compact infinite cyclic group $\group_1$ in $\C{n}$ is given by
\begin{equation}\label{eq:full group rep}
    T_{g} = F\begin{bmatrix}
    e^{-\mathrm{i}2\pi j_1 g } & & \\
    & \ddots & \\
   & &   e^{-\mathrm{i}2\pi j_J g } 
    \end{bmatrix}F^{-1}
\end{equation}
for the elements $g\in(0,1]$. This representation contains $J=n/\sum_j c_j$ distinct irreducibles given by $\rho_j(g) = e^{-\mathrm{i}2\pi j g }$ with integers $j_1< \dots< j_J$. The linear representation has the same form on $\R{n}$ as in~\Cref{eq:full group rep}, with the additional constraint that for each $j_i$ the conjugate pair $-j_i$ is also present in the linear representation (and thus $j_1 = -j_J)$. The restriction to a finite subgroup $\group_2$ with $p$ elements is given by
\begin{equation}\label{eq:subgroup rep}
    \tilde{T}_{r} = F
    \begin{bmatrix}
    e^{-\mathrm{i}2\pi j_1 \frac{r}{p} } & & \\
    & \ddots & \\
   & &   e^{-\mathrm{i}2\pi j_J \frac{r}{p} } 
    \end{bmatrix}
    F^{-1}
\end{equation}
for the elements $r=1,\dots,p$. In order to ensure that the representation of $\group_2$ has the same multiplicities than the one of $\group_1$, we need that no two distinct irreducibles of $\group_1$ are mapped to a common irreducible of the finite subgroup. For any subgroup of size $p>2j_J$, the diagonal entries in \Cref{eq:subgroup rep}, $\rho_{j}(r)=  e^{-\mathrm{i}2\pi j \frac{r}{p} }$, are orthogonal as vectors in $\C{p}$ for any $j\neq j'$ as 
\begin{equation}  
   \frac{1}{p} \sum_{r=1}^{p} \rho_{j}(r)\rho_{j'}^{*}(r) = \begin{cases}
    1 \quad \text{if} \quad j=j' \\
    0 \quad \text{ otherwise}
    \end{cases}.
\end{equation}
where $*$ denotes complex conjugation.
Thus, they are distinct irreducible representations of $\group_2$. By~\Cref{theo: irreps}, the representation in \Cref{eq:subgroup rep} is in diagonal form, and   \Cref{eq:full group rep} and \Cref{eq:subgroup rep} have the same multiplicities.
\end{proof}

\textbf{Proof of~\Cref{theo: one op}}.

\begin{proof}
\red{Our proof focuses on finite cyclic groups, whose irreducible representations have dimension $s_j=1$. If the cyclic group is infinite and compact, we can restrict its action to a finite subgroup with the same multiplicities $c_j$ of the irreducible representations using Lemma~\ref{lem: cyclic subgroup}.} Similarly to \Cref{theo: multiple op}, we have to prove that $\hat{\signalset}\setminus{\signalset}$ is empty for almost every $A\in\R{m\times n}$ if $m>2k +\max_j c_j+1$. 
In this case, this is equivalent to

\begin{align}  \label{eq:AZ}
  \underbrace{\begin{bmatrix}
  - AT_1 & AT_{1} &  & \\ 
   \vdots & & \ddots & \\ 
  - AT_{\ntransf}& &  & AT_{\ntransf}
 \end{bmatrix}}_{G_\alpha\in\R{n(\ntransf+1) \times m\ntransf}}
  \underbrace{\begin{bmatrix}
  v \\
   x_{1}\\
   \vdots \\ 
    x_{\ntransf}\\
 \end{bmatrix}}_{z\in \R{n(\ntransf+1)}} &\neq 0 \\
 G_{\alpha}(z) &\neq 0
\end{align}
for any $x_1,\dots,x_{\ntransf}\in\signalset$ and $v\in\R{n}\setminus{\signalset}$. \Cref{eq:AZ} can be rewritten as 
\begin{equation}
\label{eq: APhi}
    A \Phi_z \neq 0
\end{equation}
where 
\begin{equation} \label{eq:phiz}
    \Phi_z = [T_1(x_1-v),\dots,T_{\ntransf}(x_{\ntransf}-v)] \in \R{n\times \ntransf}.
\end{equation}

Moreover, letting $T_1 = T$ be the linear representation of the generator of the group, we can write $T_r = T^r$ where $T^r$ denotes the $r$th power of $T$. Thus, we can rewrite~\cref{eq:phiz} as 
\begin{equation} \label{eq:phiz2}
    \Phi_z = [T(x_1-v),\dots,T^{\ntransf}(x_{\ntransf}-v)] \in \R{n\times \ntransf}.
\end{equation}

Equation~\cref{eq: APhi} requires that the nullspace of $A$ does not contain the range of $\Phi_z$ for any choice of $z$. We perform a separate analysis for the cases of full rank or rank-deficient $\Phi_z$. We decompose $S= S_1 \bigcup S_2$, where $S_1$ is the set of $z$ such that $\Phi_z$ has full rank (i.e., $\rk{\Phi_z} = \min \{ n,\ntransf\}$) and $S_2$ is the set of $z$ such that $\Phi_z$ is rank-deficient.

We begin with analyzing the (simpler) full-rank case associated with $S_1$. If $\ntransf\geq n$, then $\rk{\Phi_z}=n$ and \cref{eq: APhi} implies the trivial inequality $A\neq0$ which has measure zero in $\R{m\times n}$.  If $\ntransf< n$, then, in order to apply Lemma~\ref{lemma:sauer}, we need to compute the dimension of 

\begin{equation}
    S_1 = \{ z\in (\R{n}\setminus{\signalset}) \times \signalset^{\ntransf} | \; \rk{\Phi_z} = \ntransf \}
\end{equation}
As being full-rank is an open condition, $S_1$ has the same dimension as $(\R{n}\setminus{\signalset}) \times \signalset^{\ntransf}$. The analysis of this set can be done in the same way as in the proof of~\Cref{theo: multiple op} (which we do not repeat here), where we have in both cases that almost every $A\in\R{m\times n}$ with $m> k + n/\ntransf$ verifies $A\Phi_z\neq 0$. Moreover, as $n/\ntransf \geq \max_j c_j$ for any linear representation of a cyclic group, we have that
\begin{align}
  m&> k+n/\ntransf > 2k + \max_j c_j +1 .
\end{align}
    
We now treat the rank-deficient case associated with $S_2$. Let $\Phi_{z,r} \in \R{n\times r}$ denote the first $r$ columns of $\Phi_z$. 
As no column of $\Phi_z$ can be exactly 0, we can decompose $S_2$ as
\begin{equation}
    S_2 = \bigcup_{r=0}^{\min\{n-1,\ntransf-1\}} S_{2,r}
\end{equation}
 where $S_{2,r}$ is the subset of $z\in S$ such that the first $r$ columns are linearly independent (i.e., $\rk{\Phi_{z,r}}=r$) but the first $r+1$ columns of $\Phi_z$ are rank-deficient (i.e., $\rk{\Phi_{z,r+1}}=r$).  We will consider  $A\Phi_{z,r} \neq 0 $, as this necessarily implies $A\Phi_z \neq 0$. 
In order to apply Lemma~\ref{lemma:sauer}, we need to bound the dimension of the set\footnote{Note that we do not consider $x_{r+2},\dots,x_{\ntransf}$ in $S_{2,r}$ as, for any fixed $v,x_1,\dots,x_{r+1}$, they yield the same submatrix $\Phi_{z,r}$.}
    \begin{equation}
      S_{2,r} =  \{ \tilde{z} = [v^{\top},x_1^{\top},\dots,x_{r+1}^{\top}]^{\top} \in (\R{n}\setminus{\signalset})\times  \signalset^{r+1}  | \; \rk{\Phi_{z,r}} =  \rk{\Phi_{z,r}} = r\}.
    \end{equation}
By Lemma~\ref{lem: v-variety}, the set of $v$ in $z$ verifying the constraint $\rk{\Phi_{z,r}} =  \rk{\Phi_{z,r+1}} = r$ is a variety of $\R{n}$ of dimension $d \leq r + \sum_{\mathcal{J}} c_j$ which varies smoothly as a function of $x_1,\dots,x_{r+1}$, where $c_j$ denotes the multiplicity of the $j$th irreducible representation of the group action and $\mathcal{J}$ is a subset of $r$ out of $\ntransf$ irreducibles.  
A variety consists of a finite union of irreducible\footnote{Despite having the same name, the term \emph{irreducible} variety is not related to the irreducible group representations.} varieties, where each irreducible has the structure of  a  smooth manifold~\citep{cox2013ideals}. 
Thus, we decompose $S_{2,r}$ as a finite union of irreducible varieties, $S_{2,r} \subseteq \bigcup_{\ell} S_{2,r,\ell}$, and divide each variety further as a countable union of local (bounded) neighbourhoods, yielding 
\begin{equation}
    S_{2,r}\subseteq \bigcup_{\ell} \bigcup_{i=1}^{\infty} S_{2,r,\ell,i}
\end{equation}
As the $v$-variety has a smooth dependence on $x_1,\dots,x_{r+1}$ by Lemma~\ref{lem: v-variety}, for each subset $S_{2,r,\ell,i}$, we can build a smooth locally Lipschitz mapping $f: S_{2,r,\ell,i} \subset\R{n(r+2)}\mapsto\R{n(r+2)}$  such that $f(S_{2,r,\ell,i})$ has a simple product structure $\{0\}^{n-d} \times (0,1)^d  \times (W \cap \signalset^{r+1})$ where $W$ is local  neighborhood of $\R{n(r+1)}$ (see~\citep[Chapter~7]{falconer2004fractal}). As a Lipschitz mapping does not increase the box-counting dimension of a set~\citep{robinson2010dimensions}, we have 
\begin{align}
    \bdim{S_{2,r,\ell,i}} &\leq \bdim{\{0\}^{n-d}} + \bdim{(0,1)^d} + \bdim{W \cap \signalset^{r+1}} \\
    &\leq d + k(r+1) \\ 
    & \leq r+ \sum_{j\in\mathcal{J}} c_j + k(r+1)
\end{align}
for both  bounded and conic signal sets $\signalset$. As we have $\rk{\Phi_{z,r}}=r$, according to Lemma~\ref{lemma:sauer}, for almost every $A\in\R{m\times n}$ there is no $z\in S_{2,r,\ell,i}$ which verifies $A\Phi_{z,r}\neq 0$ as long as $m$ verifies 
\begin{align} \label{eq:mbound}
    mr&>  k(r+1) + \sum_{j\in\mathcal{J}} c_j +r \\
    m&>  k (1+1/r) + \sum_{j\in\mathcal{J}} c_j/r + 1\\
    m&> 2k + \max_j c_j + 1
\end{align} 
where the last inequality is obtained by using $r \geq 1$ and uses the fact that $|\mathcal{J}|=r$. As the countable union of events of measure zero have measure  zero, then for almost every  $A\in\R{m\times n}$ with $m>2k+\max_j c_j + 1$,  all $z\in S_{2,r}$ verifies $A\Phi_z\neq 0$ for all possible ranks $r$.

\end{proof}

\textbf{Proof of~\Cref{theo: badAs}.}
\begin{proof}
We need to show that the following statements are equivalent:
\begin{enumerate}[label=(\roman*)]
    \item The range space of $(AT_g)^{\dagger}$ is the same for all $g\in\group$.
    \item $A$ is a  $\group$-equivariant map.
    \item $A$ can be decomposed as $A=\tilde{F}\Sigma F^{-1}$ where $\Sigma$ has the block-structure in~\Cref{theo: irreps}.
\end{enumerate}
We first prove (i)$\implies$(ii). Let $V\in \R{m \times n}$ be an orthogonal basis for the range space of $A^{\dagger}$, such that $A=M V^{\top}$ where $M\in\R{m\times m}$ is an invertible matrix. Due to (i), we have 
\begin{align}
    AT_g &= M Q_g V^{\top} \\
    &= \underbrace{M Q_g M^{-1}}_{U_g} MV^{\top} \\
    &= U_g A
\end{align}
where $Q_g\in \R{m\times m}$ is an orthogonal matrix and $U_g = M Q_g M^{-1}\in \R{m\times m}$ is invertible. 
The mapping $g\mapsto U_g$ is a valid linear representation (e.g., see ~\citep[Chapter~1]{serre1977linear}) as we have that for any two group elements $g$ and $g'$, $ U_{g\cdot g'} A = A T_{g\cdot g'} = AT_{g}T_{g'} = U_{g'} AT_g = U_gU_{g'} A$  and thus $U_{g\cdot g'} = U_g U_{g'}$. As we have $U_gA = AT_g$ where $U_g$ and $T_g$ are two linear representations of $\group$, $A$ is an equivariant map.

We now prove (ii)$\implies$(i). 
Using the decomposition $A = M V^{\top}$ and (ii) we have
\begin{align}
    AT_g &= \tilde{T}_g A\\
          &= \tilde{T}_g M V^{\top} \\
          &= \tilde{U}_g V^{\top}
\end{align}
where $\tilde{U}_g =\tilde{T}_g M$ is an invertible matrix. As $A$ and $AT_g$ share the same right singular vectors given by the matrix $V$, they have the same range space for all $g\in\group$.

The proof of (ii)$\iff$(iii) follows from~\Cref{theo: irreps} and is a standard result of linear representation theory which can be found in, for example, \citep{reeder2014notes} and \citep{stiefel2012group}.
\end{proof}

\vskip 0.2in

\end{document}